\documentclass{article}

\PassOptionsToPackage{table}{xcolor}
\usepackage{xurl}
\usepackage{hyperref}
\usepackage[preprint]{icml2026}
\usepackage[T1]{fontenc}
\usepackage[utf8]{inputenc}
\usepackage{microtype}
\usepackage{amsmath,amssymb,amsthm,mathtools,bm}
\usepackage{graphicx}
\usepackage{capt-of}
\usepackage{subcaption}
\usepackage{booktabs}
\usepackage{enumitem}
\usepackage[table]{xcolor}
\usepackage{algorithm}
\usepackage{placeins}
\usepackage[nameinlink,noabbrev]{cleveref}
\usepackage{tabularx}
\usepackage{array}

\newcommand{\BreakUnderscore}{\textunderscore\allowbreak}
\DeclareTextFontCommand{\texttt}{\ttfamily\let\_\BreakUnderscore}

\hypersetup{
  colorlinks=true,
  linkcolor=blue!60!black,
  citecolor=blue!60!black,
  urlcolor=blue!60!black
}

\newtheorem{theorem}{Theorem}[section]
\newtheorem{lemma}[theorem]{Lemma}
\newtheorem{proposition}[theorem]{Proposition}

\newtheorem{corollary}[theorem]{Corollary}

\newcommand{\E}{\mathbb{E}}
\newcommand{\R}{\mathbb{R}}
\newcommand{\tr}{\operatorname{tr}}

\newcommand{\Range}{\operatorname{range}}
\newcommand{\Null}{\operatorname{null}}
\newcommand{\Id}{I}

\newcommand{\Pbar}{\bar{P}}

\icmltitlerunning{From Order to Distribution}

\hypersetup{pdfauthor={Zonghuan Xu and Xingjun Ma},pdftitle={From Order to Distribution: An Exact Operator Framework for Forgetting in Continual Learning}}

\begin{document}
\twocolumn[
\icmltitle{From Order to Distribution: An Exact Operator Framework for Forgetting in Continual Learning}
\begin{icmlauthorlist}
\icmlauthor{Zonghuan Xu}{fudan}
\icmlauthor{Xingjun Ma}{fudan}
\end{icmlauthorlist}
\icmlaffiliation{fudan}{Institute of Trustworthy Embodied AI, Fudan University, Shanghai, China; Shanghai Key Laboratory of Multimodal Embodied AI, Shanghai, China}
\icmlcorrespondingauthor{Xingjun Ma}{xingjunma@fudan.edu.cn}
\vskip 0.3in
]
\printAffiliationsAndNotice{}

\begin{abstract}
A central challenge in continual learning is forgetting: the loss of performance on previously learned tasks after learning new ones. Prior theory has analyzed forgetting under random orderings of fixed task collections in overparameterized linear regression. We shift the focus from task order to task distribution, asking how its structure determines forgetting. In the linear setting with a shared solution, i.i.d.\ task sampling, and sequential exact fitting, we derive an exact operator identity expressing historical forgetting directly in terms of the task distribution. Building on this identity, we establish an exponential decay guarantee for expected historical forgetting under every fixed task distribution in finite dimensions, characterize its asymptotic behavior, and relate decay to the distribution's coverage of observable directions. For an individual learned task, we show that subsequent tasks can collectively support recovery without exact revisits. We derive a lower bound on recovery time and construct a task distribution attaining its inverse-coverage scaling.

\end{abstract}

\section{Introduction}
In continual learning, a model learns tasks sequentially, and its performance on earlier tasks may deteriorate as new tasks are learned. A natural theoretical question is what determines this forgetting. In the exact-fit overparameterized linear regime, prior work analyzes given task collections under different presentation orders and establishes corresponding convergence guarantees, including i.i.d.\ uniform sampling from finite collections \citep{evron2022catastrophic,evron2026sgd}. We take the task distribution as the central object of analysis and ask how its structure determines the magnitude, persistence, and eventual decay of forgetting.

Concretely, we study overparameterized linear regression under a shared-solution assumption: every task $(X_t,y_t)$ is consistent with a common parameter $w^\star$, and tasks are drawn i.i.d.\ from a distribution $\Pi$. Starting from $w_0=0$, learning selects the exact solution of each task closest to the previous iterate. This update is the limiting solution selected by least-squares gradient flow initialized at that iterate. We study the resulting historical forgetting
\[
\begin{gathered}
F^\Pi(k):=\mathbb E\!\left[\frac1k\sum_{t=1}^k\ell_t(w_k)\right],\\
\ell_t(w):=\frac1{n_t}\|X_tw-y_t\|_2^2.
\end{gathered}
\]
Each task has zero loss immediately after it is learned, so its later loss directly measures forgetting.

The distributional perspective lets us ask which directions tasks repeatedly constrain, which receive little information, and how these differences enter historical losses. It also describes recovery: subsequent tasks may supply information relevant to an earlier task, allowing its performance to improve even without repeating that task.

A key starting point in prior analysis is the projection dynamics
\[
w_t-w^\star=P_t(w_{t-1}-w^\star),\qquad P_t=I-X_t^+X_t.
\]
This describes how each exact-fit update changes the parameter error. Writing $e_k=w_k-w^\star$, the historical loss is
\[
\ell_t(w_k)=e_k^\top C_te_k,\qquad C_t=\frac1{n_t}X_t^\top X_t.
\]
The error $e_k$ contains the projection $P_t$, while $C_t$ depends on the same $X_t$. Independence between task draws therefore does not justify replacing $\mathbb E[e_k^\top C_te_k]$ by $\operatorname{tr}(\mathbb E[C_t]\,\mathbb E[e_ke_k^\top])$. The latter calculation applies to independent evaluation and generally does not give historical loss.

Projection-based residual bounds connect these quantities, for example through
\[
\|X_t(w-w^\star)\|_2^2
\le\|X_t\|_{\mathrm{op}}^2\,\|(I-P_t)(w-w^\star)\|_2^2.
\]
Such bounds support uniform forgetting guarantees, but do not retain the full covariance that determines the loss. Rescaling $(X_t,y_t)$ to $(\varepsilon X_t,\varepsilon y_t)$, for $\varepsilon>0$, leaves the projection and exact-fit update unchanged while multiplying the task loss by $\varepsilon^2$. An exact characterization must therefore connect the update geometry to the loss itself.

The i.i.d.\ structure makes this connection tractable. Conditional on the evaluated task, the preceding and subsequent updates are independent and can be averaged through the same distribution-induced operator. This yields an exact identity for $F^\Pi(k)$, separating temporal evolution from coefficients determined by task covariance and target direction. We further show that historical forgetting can differ even when the loss normalization and the entire second-moment operator are fixed. The additional task information retained by the identity thus changes the quantity being studied.

We use this framework to analyze both the decay of average historical forgetting and the recovery of an individual learned task. The former connects distributional coverage to long-run contraction; the latter explains how subsequent tasks jointly supply information for recovery and how weak support can limit its speed.

\paragraph{Our Contributions.}
We make three main contributions.
\begin{enumerate}[leftmargin=1.5em]
\item \textbf{An exact operator framework.}
We derive an exact identity and spectral expansion for historical forgetting. Comparisons at a fixed second-moment operator identify additional task information needed to determine forgetting, with exact extremizers in a two-dimensional task family.
\item \textbf{Convergence rates and distribution geometry.}
We characterize asymptotic behavior, establish an exponential bound through coverage of observable directions, and characterize zero forgetting for all shared targets through almost-sure commutativity.
\item \textbf{Recovery through other tasks.}
We show that a learned task can recover without exact revisits, bound the required time using distributional coverage, and construct a family attaining the inverse-coverage scaling.
\end{enumerate}

\section{Setup}\label{sec:setup}

We work with regression tasks \((X,y)\), where \(X\in\mathbb{R}^{n\times d}\) and \(y\in\mathbb{R}^n\), in the overparameterized regime \(\operatorname{rank}(X)<d\). As in prior linear analyses, we assume that the tasks share a common solution: there exists \(w^\star\in\mathbb{R}^d\) such that \(y=Xw^\star\) for every task under consideration.

Given a task \((X,y)\) and an initialization \(u\in\mathbb{R}^d\), the task-wise exact-fit update is defined by \(\mathcal S_{X,y}(u):=\arg\min_{w\in\mathbb{R}^d}\frac12\|w-u\|_2^2\) subject to \(Xw=y\). A standard projection calculation gives
\[
\mathcal S_{X,y}(u)=u+X^+(y-Xu).
\]
Under the common-solution assumption \(y=Xw^\star\), this becomes
\[
\mathcal S_{X,y}(u)=w^\star+P(u-w^\star),
\qquad
P:=I-X^+X.
\]

Tasks arrive as an i.i.d.\ stream from a task distribution \(\Pi\): \((X_1,y_1),(X_2,y_2),\ldots\stackrel{\mathrm{i.i.d.}}{\sim}\Pi\). Starting from \(w_0=0\), the sequential exact-fit iterates are defined by \(w_t:=\mathcal S_{X_t,y_t}(w_{t-1})\) for \(t\ge 1\). For convenience, we write \(P_t:=I-X_t^+X_t\), \(C_t:=\frac{1}{n_t}X_t^\top X_t\), and \(X_\star:=w^\star w^{\star\top}\). Since \(P_t\) projects orthogonally onto \(\Null(X_t)\), we have \(C_tP_t=P_tC_t=0\).

Our object of study is the forgetting quantity \(F^\Pi(k):=\mathbb{E}\!\left[\frac1k\sum_{t=1}^k \ell_t(w_k)\right]\), where \(\ell_t(w):=\frac1{n_t}\|X_t w-y_t\|_2^2\). We assume throughout that the expectations appearing below are finite; in particular, \(\mathbb E\|C_t\|_F<\infty\).

\section{Exact Dynamics for Forgetting}\label{sec:exact-dynamics}
We first derive an exact identity at every finite horizon, then separate the factors governing the magnitude and temporal evolution of forgetting. Finally, we compare distributions with the same complete second-moment operator to identify information that parameter-error dynamics alone do not determine.

\begin{theorem}[Exact forgetting identity]\label{thm:exact-forgetting-identity}
Let \(X_\star:=w^\star w^{\star\top}\) and \(C_t:=\frac{1}{n_t}X_t^\top X_t\). For every symmetric matrix \(A\), define \(S_\Pi(A):=\mathbb E[P_t A P_t]\). This operator is independent of \(t\) because the task sequence is i.i.d. Under the common-solution assumption, for every \(k\ge 2\),
\[
F^\Pi(k)=\frac1k\sum_{t=1}^{k-1}\mathbb E\tr\!\Big(S_\Pi^{k-t}(C_t)\,P_t\,S_\Pi^{t-1}(X_\star)\,P_t\Big).
\]
\end{theorem}

\paragraph{Proof sketch.}
Let $e_t=w_t-w^\star$. Exact fitting gives $e_t=P_te_{t-1}$, with $e_0=-w^\star$. Fix a historical position $t$ and write $R_{t,k}=P_k\cdots P_{t+1}$ and $L_{t-1}=P_{t-1}\cdots P_1$, with $L_0=I$. Then
\[
\ell_t(w_k)=\tr\!\left[R_{t,k}^\top C_tR_{t,k}\,P_tL_{t-1}X_\star L_{t-1}^\top P_t\right].
\]
Conditional on task $t$, the preceding and subsequent task sequences are independent. Averaging them separately gives
\[
\begin{aligned}
\mathbb E[R_{t,k}^\top C_tR_{t,k}\mid X_t,y_t]&=S_\Pi^{k-t}(C_t),\\
\mathbb E[L_{t-1}X_\star L_{t-1}^\top]&=S_\Pi^{t-1}(X_\star).
\end{aligned}
\]
Consequently,
\[
\mathbb E\ell_t(w_k)=\mathbb E\tr\!\left[S_\Pi^{k-t}(C_t)\,P_tS_\Pi^{t-1}(X_\star)P_t\right].
\]
Averaging over historical positions and using $\ell_k(w_k)=0$ proves the identity. Appendix~\ref{app:proof-exact-forgetting-identity} gives the full proof.

The three factors distinguish preceding learning, the evaluated task, and subsequent learning. The matrix $S_\Pi^{t-1}(X_\star)$ is the error second moment before task $t$, and the two factors $P_t$ account for fitting that task. The term $S_\Pi^{k-t}(C_t)$ averages how later updates affect its loss. Keeping task covariance and training projection inside the same expectation allows the task information in the updates to enter the final historical loss. The spectral expansion below separates this information from the temporal factors.

\begin{theorem}[Exact spectral expansion]\label{thm:exact-spectral-expansion}
Let \(\mathbb S^d:=\{A\in\mathbb R^{d\times d}:A^\top=A\}\), equipped with the Frobenius inner product. Let \(\mathbb S^d=\mathcal E_1\oplus\cdots\oplus\mathcal E_R\) be the orthogonal decomposition of \(\mathbb S^d\) into eigenspaces of \(S_\Pi\), and let \(\rho_1,\dots,\rho_R\) be the corresponding spectral levels, so that \(S_\Pi(Y)=\rho_r Y\) for every \(Y\in\mathcal E_r\). Let \(\mathcal Q_r:\mathbb S^d\to\mathcal E_r\) be the Frobenius-orthogonal projector onto \(\mathcal E_r\). For each \(r,s\in\{1,\dots,R\}\), define the spectral coupling coefficient
\[
\beta_{rs}:=\mathbb E\tr\!\big(\mathcal Q_r(C_t)\,P_t\,\mathcal Q_s(X_\star)\,P_t\big),
\]
which measures how the \(r\)-th spectral component of the task covariance couples to the \(s\)-th spectral component of the shared target inside the loss, and is independent of \(t\). For every \(k\ge 2\),
\[
F^\Pi(k)=\frac1k\sum_{r,s=1}^R \beta_{rs}\,\Gamma_{rs}(k).
\]
Moreover, \(\Gamma_{rs}(k)=(k-1)\rho_r^{\,k-1}\) if \(r=s\), and \(\Gamma_{rs}(k)=\rho_r\frac{\rho_r^{\,k-1}-\rho_s^{\,k-1}}{\rho_r-\rho_s}\) if \(r\neq s\).
\end{theorem}

\paragraph{Proof sketch.}
The operator $S_\Pi$ is self-adjoint and positive semidefinite on $\mathbb S^d$, so $S_\Pi^m(A)=\sum_r\rho_r^m\mathcal Q_r(A)$. Substituting both spectral expansions into Theorem~\ref{thm:exact-forgetting-identity} gives
\[
F^\Pi(k)=\frac1k\sum_{r,s}\beta_{rs}\sum_{t=1}^{k-1}\rho_r^{k-t}\rho_s^{t-1}.
\]
Evaluating the geometric sums proves the stated formula; see Appendix~\ref{app:proof-exact-spectral-expansion}.

The spectral levels determine the available temporal factors, while $\beta_{rs}$ determines how they enter actual loss. Equal spectral levels need not imply equal forgetting curves. Even fixing the entire operator can leave forgetting undetermined through these coefficients.

\paragraph{Identical second-moment dynamics and different forgetting.}
The parameter-error second moment satisfies
\[
\mathbb E[e_ke_k^\top]=S_\Pi^k(X_\star).
\]
An independent test task has expected loss $\tr(\bar C_\Pi S_\Pi^k(X_\star))$, where $\bar C_\Pi=\mathbb E_\Pi C$. Thus $S_\Pi$, $\bar C_\Pi$, and the target determine both quantities at every step. Historical forgetting can still vary.

\begin{proposition}[Forgetting differences and extremizers at a fixed operator]\label{prop:main-same-operator}
There exist two task laws $\Pi_A,\Pi_B$ in dimension two with common target $w^\star=e_1$, zero initialization, and unit-norm one-row designs, such that
\[
S_{\Pi_A}=S_{\Pi_B},\qquad
\bar C_{\Pi_A}=\bar C_{\Pi_B}=\tfrac12 I.
\]
Their parameter-error second moments and independent-task risks agree at every step, but
\[
F^{\Pi_A}(k)<F^{\Pi_B}(k)\qquad(k\ge2).
\]
In particular, $F^{\Pi_A}(2)=9/512$ and $F^{\Pi_B}(2)=15/512$. Moreover, every two-dimensional task law $\Pi$ with the same operator, target, and unit-norm one-row designs satisfies
\[
F^{\Pi_A}(k)\le F^\Pi(k)\le F^{\Pi_B}(k)\qquad(k\ge1).
\]
The same two laws therefore attain the minimum and maximum at every horizon, and also minimize and maximize peak historical forgetting.
\end{proposition}

Both laws use three task types. Under $\Pi_A$, the nullspace projector takes the values
\[
P_0=\begin{pmatrix}1&0\\0&0\end{pmatrix},\qquad
P_\pm=\frac18\begin{pmatrix}1&\pm\sqrt7\\\pm\sqrt7&7\end{pmatrix},
\]
with probabilities $3/7,2/7,2/7$. Under $\Pi_B$, replace each projector by $I-P$ and keep the probabilities. Choose a unit row $x$ with $x^\top x=I-P$ and label $y=xw^\star$. Since $\mathbb E_{\Pi_A}P=I/2$, complementation preserves the full operator:
\[
\begin{aligned}
S_{\Pi_B}(A)
&=\mathbb E_{\Pi_A}[(I-P)A(I-P)]\\
&=A-\tfrac12 A-\tfrac12 A+S_{\Pi_A}(A)
=S_{\Pi_A}(A).
\end{aligned}
\]
Substitution into the exact identity gives the strict forgetting difference; Proposition~\ref{prop:appendix-same-operator} provides its closed form.

\paragraph{Why the same laws are extremal at every horizon.}
For this coordinate target, the remaining freedom can be expressed by a single scalar moment. Write an arbitrary rank-one projector as
\[
P=\frac12\begin{pmatrix}1+z&v\\v&1-z\end{pmatrix},
\qquad z^2+v^2=1.
\]
The common operator fixes the first and second direction moments, including $\mathbb Ez=0$ and $\mathbb Ez^2=s$, where $s=3/4$. It does not fix $\mu_\Pi:=\mathbb Ez^3$. Let $c=1/2$, $a=3/8$, and $b=1/8$. With $Z=\operatorname{diag}(1,-1)$, the shared operator gives
\[
S^j(X_\star)=\tfrac12(c^j I+a^j Z).
\]
For a task of age $m=k-t$, the rank-one relation $PBP=\tr(PB)P$ reduces the conditional identity to
\[
\begin{aligned}
\mathbb E\ell_t(w_{t+m})
&=\frac14\mathbb E\!\left[
(c^{t-1}+a^{t-1}z)\right.\\[-2pt]
&\hspace{42pt}\left.\cdot(c^m-a^mz^2-b^m(1-z^2))\right].
\end{aligned}
\]
Thus all distributional freedom in this loss enters through $\mu_\Pi$. Averaging over historical positions yields
\[
F^\Pi(k)=F_0(k)-\mu_\Pi q_k,
\]
where the operator fixes both
\[
\begin{aligned}
F_0(k)&=\frac1{4k}\sum_{m=1}^{k-1}c^{k-m-1}
\left[c^m-sa^m-(1-s)b^m\right],\\
q_k&=\frac1{4k}\sum_{m=1}^{k-1}a^{k-m-1}(a^m-b^m)>0.
\end{aligned}
\]
The positive coefficient $q_k$ has the same sign at every horizon. It remains to find the exact range of $\mu_\Pi$. Since $-1\le z\le1$,
\[
\begin{aligned}
0&\le\mathbb E[(1-z)(z+s)^2]=s(1-s)-\mu_\Pi,\\
0&\le\mathbb E[(1+z)(z-s)^2]=s(1-s)+\mu_\Pi.
\end{aligned}
\]
Consequently, $|\mu_\Pi|\le s(1-s)=3/16$. The displayed law $\Pi_A$ attains $\mu_{\Pi_A}=3/16$, while complementation changes $z$ to $-z$, so $\mu_{\Pi_B}=-3/16$. This proves both extremal bounds for all $k$ simultaneously. At $k=2$, $F_0(2)=12/512$ and $(3/16)q_2=3/512$, giving the stated endpoints. The independent-test risk, meanwhile, is $c^{k+1}=2^{-(k+1)}$ under every law with this operator.

The calculation identifies where the extra moment enters: the evaluated direction appears both in the residual left by fitting the task and in the loss after subsequent updates. Their product retains a third direction moment even though the error second-moment operator fixes only the first two moments of $(z,v)$. Theorem~\ref{thm:fixed-operator-range} extends the argument to a family of fixed operators and arbitrary target directions. There the remaining moments fill an exact disk, and optimizing the loss over that disk gives all attainable endpoints. Corollary~\ref{cor:fixed-operator-extremizers} gives the simultaneous-extremizer conclusion used here.

For $\Pi_\theta=\theta\Pi_A+(1-\theta)\Pi_B$, $0\le\theta\le1$, the operator and mean covariance remain fixed, while the identity gives
\[
F^{\Pi_\theta}(k)=\theta F^{\Pi_A}(k)+(1-\theta)F^{\Pi_B}(k).
\]
Every value between the two endpoints is therefore attainable at each horizon. Even complete agreement of error second moments and independent-test risk at every step leaves historical retention undetermined. The task information in the exact identity both computes this difference and identifies its extremal range in the stated family.

The ordering also holds before averaging over the training history. For any learning position $t$ and age $m\ge1$, the coefficient of $-\mu_\Pi$ in the displayed conditional loss is $a^{t-1}(a^m-b^m)/4>0$. Thus the same two laws minimize and maximize $\mathbb E\ell_t(w_{t+m})$ for every position and age. The separation is not caused by assigning different weights to older tasks: it is already present in their expected losses at matched evaluation times. All comparisons keep the target and loss normalization fixed.

\section{Convergence Rates and Their Interpretation}\label{sec:rates}

We use the spectral expansion to characterize the leading asymptotic behavior of forgetting, establish a data-dependent upper bound, and analyze the geometric structure governing decay and zero forgetting.

To identify the asymptotically leading contribution, define
\[
\rho_\Pi:=\|S_\Pi\|_{\mathrm{op},F},
\]
and let \(\mathcal E_\Pi:=\{Y\in\mathbb S^d:S_\Pi(Y)=\rho_\Pi Y\}\) denote the top eigenspace of \(S_\Pi\). Let \(\mathcal Q_\Pi:\mathbb S^d\to\mathcal E_\Pi\) be the Frobenius-orthogonal projector onto \(\mathcal E_\Pi\), let \(\eta_\Pi:=\|S_\Pi|_{\mathcal E_\Pi^\perp}\|_{\mathrm{op},F}\), and define the top-level activation coefficient by \(c_\Pi^{\mathrm{top}}(w^\star):=\mathbb E\tr\!\big(\mathcal Q_\Pi(C_t)\,P_t\,\mathcal Q_\Pi(X_\star)\,P_t\big)\).

\begin{theorem}[Asymptotic characterization of forgetting]\label{thm:leading-term}
Assume that \(\eta_\Pi<\rho_\Pi\). Then, as \(k\to\infty\),
\[
F^\Pi(k)=c_\Pi^{\mathrm{top}}(w^\star)\rho_\Pi^{\,k-1}+O\!\left(\frac{\rho_\Pi^{\,k-1}}{k}+\eta_\Pi^{\,k-1}\right).
\]
\end{theorem}

\paragraph{Proof sketch.}
Order the distinct spectral levels as $\rho_1=\rho_\Pi>\rho_2>\cdots\ge0$. In Theorem~\ref{thm:exact-spectral-expansion}, the highest-level diagonal term contributes
\[
\frac1k\beta_{11}\Gamma_{11}(k)
=\frac{k-1}{k}c_\Pi^{\mathrm{top}}(w^\star)\rho_\Pi^{k-1}.
\]
For $s\ge2$, the geometric sum satisfies
\[
\Gamma_{1s}(k)=\rho_\Pi\frac{\rho_\Pi^{k-1}-\rho_s^{k-1}}{\rho_\Pi-\rho_s}
=O(\rho_\Pi^{k-1}),
\]
and the same bound holds for $\Gamma_{r1}$ with $r\ge2$. The prefactor $1/k$ makes these cross terms $O(\rho_\Pi^{k-1}/k)$. Terms involving only lower levels contribute $O(\eta_\Pi^{k-1})$. Adding the terms proves the expansion; Appendix~\ref{app:proof-leading-term} gives the full proof.

When $c_\Pi^{\mathrm{top}}(w^\star)>0$ and $\rho_\Pi<1$, the expansion gives $F^\Pi(k)\sim c_\Pi^{\mathrm{top}}(w^\star)\rho_\Pi^{k-1}$. This distinguishes a slow mode present in the operator from one contributing to historical loss. If the highest-level coefficient vanishes, the full expansion must be examined: cross terms can still retain the scale $\rho_\Pi^{k-1}/k$.

A simple example illustrates why this distinction matters. Suppose a task constrains either coordinate in dimension two, with equal probabilities. Its nullspace projector is either $\operatorname{diag}(1,0)$ or $\operatorname{diag}(0,1)$. The error second-moment operator has spectral radius $1/2$: a coordinate not yet observed can retain error. Yet an observed coordinate remains fitted forever, so historical forgetting is zero for every target and horizon. Parameter error can therefore contract at a nonzero spectral rate while contributing no historical loss. Proposition~\ref{prop:commuting-zero-forgetting} characterizes this zero-forgetting structure more generally.

An upper bound at every finite horizon does not require this coefficient to be nonzero.

\begin{theorem}[Data-dependent exponential upper bound]\label{thm:upper-bound}
For every \(k\ge 2\),
\[
F^\Pi(k)\le \frac{k-1}{k}\,\rho_\Pi^{\,k-1}\,\|X_\star\|_F\,\mathbb E\|C_t\|_F.
\]
\end{theorem}

\paragraph{Proof sketch.}
For symmetric matrices $A,B$ and an orthogonal projector $P$,
\[
|\tr(APBP)|\le\|A\|_F\,\|PBP\|_F\le\|A\|_F\,\|B\|_F.
\]
Applying this to Theorem~\ref{thm:exact-forgetting-identity} gives
\[
F^\Pi(k)\le\frac1k\sum_{t=1}^{k-1}\mathbb E\!\left[
\|S_\Pi^{k-t}(C_t)\|_F\,
\|S_\Pi^{t-1}(X_\star)\|_F\right].
\]
Since $\|S_\Pi^m(A)\|_F\le\rho_\Pi^m\|A\|_F$, each term is at most $\rho_\Pi^{k-1}\|X_\star\|_F\mathbb E\|C_t\|_F$. Averaging over the $k-1$ historical positions proves the claim. Appendix~\ref{app:proof-upper-bound} gives the full proof.

Theorem~\ref{thm:upper-bound} separates a contraction factor from the scale of the loss. The projector distribution determines \(\rho_\Pi\), while \(\mathbb E\|C_t\|_F\) and \(\|X_\star\|_F=\|w^\star\|_2^2\) control the prefactor. Rescaling all tasks by the same factor changes the loss scale while preserving the update geometry. To obtain a strict contraction factor for every fixed distribution, we restrict the dynamics to the directions visible to the tasks.

\begin{corollary}[Exponential decay on the observable subspace]\label{cor:observable-decay}
Let \(H_\Pi=\mathbb E[I-P_t]\), \(N_\Pi=\ker H_\Pi\), and \(V_\Pi=N_\Pi^\perp\), and let \(R_\Pi\) be the orthogonal projector onto \(V_\Pi\). If \(V_\Pi=\{0\}\), then \(F^\Pi(k)=0\) for every \(k\). Otherwise, define
\[
\begin{aligned}
S_{\Pi,V}(A)&=\mathbb E[(P_t|_{V_\Pi})A(P_t|_{V_\Pi})],\\
\rho_{\Pi,V}&=\|S_{\Pi,V}\|_{\mathrm{op},F},
\qquad g_\Pi=\lambda_{\min}(H_\Pi|_{V_\Pi}).
\end{aligned}
\]
Then \(g_\Pi>0\) and \(\rho_{\Pi,V}\le1-g_\Pi<1\). With \(K_\Pi=\|R_\Pi w^\star\|_2^2\,\mathbb E\|C_t\|_F\), for every \(k\ge2\),
\[
\begin{aligned}
F^\Pi(k)&\le\frac{k-1}{k}\rho_{\Pi,V}^{\,k-1}K_\Pi\\
&\le K_\Pi(1-g_\Pi)^{k-1}.
\end{aligned}
\]
\end{corollary}

\paragraph{Proof sketch.}
The identity \(v^\top H_\Pi v=\mathbb E\|(I-P_t)v\|_2^2\) shows that \(N_\Pi\) is the common nullspace of the tasks, almost surely. Each \(P_t\) preserves \(V_\Pi\), and each \(C_t\) annihilates \(N_\Pi\), so restricting to \(V_\Pi\) leaves every historical loss unchanged. On this finite-dimensional space, \(H_\Pi\) is positive definite. For symmetric \(A\) on \(V_\Pi\),
\[
\begin{aligned}
\langle A,S_{\Pi,V}(A)\rangle_F
&\le\tr\!\big(A^2\mathbb E[P_t|_{V_\Pi}]\big)\\
&\le(1-g_\Pi)\|A\|_F^2.
\end{aligned}
\]
Theorem~\ref{thm:upper-bound} applied to the restricted dynamics proves the claim. A full proof is given in Appendix~\ref{app:proof-observable-decay}.

Thus every fixed task distribution in the finite-dimensional setting admits an exponential forgetting bound. Errors in the common nullspace never enter task losses, even though they prevent strict contraction on the full space. Removing these directions exposes the positive coverage gap \(g_\Pi\): it measures average task coverage in the least-covered observable direction, and gives a guaranteed decay timescale of order \(1/g_\Pi\) when the gap is small. Actual forgetting can decay faster; this bound alone determines neither its peak nor its magnitude. The geometric characterizations below apply equally to \(S_{\Pi,V}\), describing contraction on the observable subspace.

To interpret the rate-controlling quantity \(\rho_\Pi\), it is useful to measure how much a nonzero symmetric matrix \(A\) is preserved, on average, after one random task update. We therefore define its normalized second-order invisibility score by
\[
\mathcal I_\Pi(A):=\frac{\mathbb E\|P_tAP_t\|_F^2}{\|A\|_F^2}.
\]

\begin{proposition}[Variational characterization of \(\rho_\Pi\)]\label{prop:variational-rho}
For every nonzero \(A\in\mathbb S^d\), one has \(0\le \mathcal I_\Pi(A)\le 1\). Moreover,
\[
\rho_\Pi=\max_{A\in\mathbb S^d,\;A\neq 0}\mathcal I_\Pi(A).
\]
If \(A=UU^\top\) for some \(U\in\mathbb R^{d\times m}\), then
\[
\mathcal I_\Pi(UU^\top)=\frac{\mathbb E\|U^\top P_tU\|_F^2}{\|UU^\top\|_F^2}.
\]
In particular, if \(E\subset\mathbb R^d\) is an \(m\)-dimensional subspace and \(\Pi_E\) is the orthogonal projector onto \(E\), then
\[
\begin{aligned}
\mathcal I_\Pi(\Pi_E)&=\frac1m\,\mathbb E\|\Pi_E P_t\Pi_E\|_F^2\\
&=\frac1m\,\mathbb E\sum_{j=1}^m \cos^4\!\big(\theta_j(E,\Null(X_t))\big),
\end{aligned}
\]
where \(\theta_1(E,\Null(X_t)),\dots,\theta_m(E,\Null(X_t))\) are the principal angles between \(E\) and \(\Null(X_t)\).
\end{proposition}

\paragraph{Proof sketch.}
The operator $S_\Pi$ is self-adjoint, and
\[
\langle A,S_\Pi(A)\rangle_F=\mathbb E\|P_tAP_t\|_F^2.
\]
Thus $\mathcal I_\Pi(A)$ is its Rayleigh quotient, and the variational principle gives the first claim. For $A=UU^\top$, cyclicity of trace gives $\|P_tUU^\top P_t\|_F^2=\|U^\top P_tU\|_F^2$. If the columns of $U$ form an orthonormal basis for $E$, the eigenvalues of $U^\top P_tU$ are the squared cosines of the principal angles. This proves the subspace formula; see Appendix~\ref{app:proof-variational-rho}.

When some error matrix changes little under typical task updates, $\rho_\Pi$ is close to one. The subspace formula gives a geometric interpretation: if a subspace is usually close to task nullspaces, errors on that subspace are difficult to remove. These slow modes lead to slowly decaying historical forgetting when they contribute to the loss through the coefficients in Theorem~\ref{thm:exact-spectral-expansion}.

The isotropic fixed-rank model makes this concrete. If every task observes a uniformly random $r$-dimensional subspace, then
\[
H_\Pi=\frac rd I,\qquad \rho_\Pi=1-\frac rd.
\]
Indeed, $S_\Pi(I)=(1-r/d)I$, while $\langle A,S_\Pi(A)\rangle_F\le(1-r/d)\|A\|_F^2$. Its exponential timescale is
\[
\frac1{-\log(1-r/d)}\sim\frac dr\qquad(r/d\to0).
\]
In this model, fewer constraints per task relative to the ambient dimension mean a longer learning process to remove error. For general distributions, tasks of the same rank can repeatedly cover similar directions or cover complementary ones, with different contraction rates. Dimension and task rank must therefore be considered together with the distribution's coverage geometry.

The exponential guarantee is specific to each fixed distribution. Its coverage gap can become arbitrarily small across a family of distributions or as dimension grows. When $\rho_{\Pi,V}$ is close to one, the bound evolves on the scale $k(1-\rho_{\Pi,V})$; the ratio $kr/d$ describes this scale when the contraction gap is comparable to $r/d$. This distinguishes the fixed-distribution limit from guarantees uniform over task collections \citep{evron2022catastrophic,evron2026sgd}. A low-rank task can be easy to interpolate because it leaves many parameter directions unconstrained. Eliminating the remaining error then depends on which complementary constraints later tasks supply. Coverage makes this collective identification explicit: it records both the available directions and the probability of learning them.

Task geometry also determines whether forgetting occurs at any finite horizon.

\begin{proposition}[Zero forgetting for all shared targets]\label{prop:commuting-zero-forgetting}
Fix the marginal distribution of \(X\) under the assumptions of Section~\ref{sec:setup}, and for each \(u\in\mathbb R^d\), let \(\Pi_u\) be the task law obtained by setting \(y=Xu\), with initialization \(w_0=0\). The following conditions are equivalent:
\begin{enumerate}[label=(\roman*),leftmargin=2.2em]
\item For two independent draws \(P,P'\) from the induced null-space projector law, \(PP'=P'P\) almost surely.
\item \(F^{\Pi_u}(2)=0\) for every \(u\in\mathbb R^d\).
\item \(F^{\Pi_u}(k)=0\) for every \(u\in\mathbb R^d\) and every \(k\ge 1\).
\end{enumerate}
\end{proposition}

\paragraph{Proof sketch.}
For any fixed \(k\), condition (i) makes the sampled projectors pairwise commuting almost surely. Since \(e_k:=w_k-u=-P_k\cdots P_1u\), moving any \(P_t\) to the front shows that \(e_k\in\Null(X_t)\) for every \(t\le k\), proving (iii); (iii) immediately implies (ii). Conversely, exact fitting gives
\[
F^{\Pi_u}(2)=\frac12\mathbb E\!\left[\frac1{n_1}\|X_1P_2P_1u\|_2^2\right].
\]
Under (ii), applying this identity to each standard basis vector and intersecting the finitely many probability-one events yields \(X_1P_2P_1=0\) almost surely. Multiplying by \(X_1^+\) gives \((I-P_1)P_2P_1=0\), hence \(P_2P_1=P_1P_2P_1\); taking transposes yields \(P_1P_2=P_1P_2P_1\), proving (i).
A full proof is given in Appendix~\ref{app:proof-commuting-zero-forgetting}.

Proposition~\ref{prop:commuting-zero-forgetting} gives a distributional formulation of the two-task zero-forgetting criterion of \citet[Theorem~6]{evron2022catastrophic}, showing that almost-sure commutativity characterizes zero forgetting for every shared target at every horizon. Failure of almost-sure commutativity implies that some shared target exhibits strictly positive expected forgetting after just two tasks. As a consequence, a positive lower bound cannot hold uniformly over all task distributions and shared targets.

\section{Task Distributions, Forgetting, and Recovery}\label{sec:experiments}\label{sec:recovery}

Previous sections average forgetting over historical tasks. We now fix a task that has already been learned and ask how subsequent learning disrupts and restores its performance. Temporary deterioration followed by recovery has been observed in continual learning \citep{delange2023stability}. We use the exact identity to study how subsequent tasks jointly support this recovery and what limits its speed. A general lower bound connects recovery time to distributional coverage; an exact family attains its inverse-coverage scaling.

\subsection{Recovery through the task distribution}\label{sec:distribution-recovery}

Under the same shared-solution and exact-fit assumptions, first learn $T$ independent tasks from $\Pi$, fit a prescribed task $\tau$, and then resume independent sampling from $\Pi$. The prescribed task is fixed independently of the sampled history. Define
\[
G_m=\mathbb E\ell_\tau(w_{T+1+m}),\qquad m\ge0,
\]
where the expectation averages both the preceding and subsequent task sequences. Exact fitting gives $G_0=0$, so later loss directly measures forgetting of $\tau$. Fixing the evaluated task in Theorem~\ref{thm:exact-forgetting-identity} gives
\begin{equation}\label{eq:conditional-recovery-identity}
\begin{aligned}
A_0&=P_\tau S_\Pi^T(X_\star)P_\tau,\\
G_m&=\tr\!\left(C_\tau S_\Pi^m(A_0)\right).
\end{aligned}
\end{equation}
The prehistory determines which errors remain, fitting $\tau$ removes their contribution to its current loss, and subsequent updates can transfer residual error back into that loss. The identity retains all three parts of this dependence. Averaging over $\tau\sim\Pi$ and its position recovers the historical forgetting studied earlier.

Recall $H_\Pi=\mathbb E[I-P]$ and the observable space $V_\Pi$ from Corollary~\ref{cor:observable-decay}. For a unit direction $u$,
\[
u^\top H_\Pi u=\mathbb E\|(I-P)u\|^2
\]
measures how much a background update covers that direction. It includes contributions from every task with a component along $u$. For a general task, write
\[
\begin{aligned}
h_\tau&=\|C_\tau^{1/2}H_\Pi C_\tau^{1/2}\|_{\mathrm{op}},\\
E_m&=\mathbb E\|w_{T+1+m}-w^\star\|^2.
\end{aligned}
\]
For $C_\tau=uu^\top$ with $\|u\|=1$, $h_\tau=u^\top H_\Pi u$. For a multidimensional task it measures the largest coverage after weighting directions by their contribution to task loss.

For a finite background with probabilities $p_i$, this directional support is $\sum_i p_i\|(I-P_i)u\|^2$. Every task constraining a component of $u$ contributes to it, even if none has the same input subspace as $\tau$. The condition $\Range(C_\tau)\subseteq V_\Pi$ below means that the directions evaluated by the old task lie in the span collectively observed by the background. The probabilities then determine how frequently this available information is supplied, which separates eventual recovery from its speed.

\begin{theorem}[Distribution support and recovery]\label{thm:distribution-recovery}
Suppose $C_\tau\ne0$ and $\Range(C_\tau)\subseteq V_\Pi$. Then $G_m\to0$, even if the background distribution never samples $\tau$ itself. For $G_m>0$ and $0<\varepsilon<1$, let $R_\varepsilon(m)$ be the smallest number of further updates after which the loss remains at most $\varepsilon G_m$. Then
\begin{equation}\label{eq:recovery-time-lower}
R_\varepsilon(m)\ge
\left\lceil\frac{(1-\sqrt\varepsilon)^2G_m}{h_\tau E_m}\right\rceil.
\end{equation}
\end{theorem}

\paragraph{Proof sketch.}
The background contracts error on $V_\Pi$, which gives recovery. Conditional on the current error, its expectation after $s$ further updates is $(I-H_\Pi)^s e$. Conditional Jensen and a bound on this displacement yield
\begin{equation}\label{eq:recovery-displacement}
G_{m+s}\ge\left(\sqrt{G_m}-\sqrt{s h_\tau E_m}\right)_+^2.
\end{equation}
This limits how quickly loss can fall and implies~\eqref{eq:recovery-time-lower}. Appendix~\ref{app:proof-recovery} gives the full proof and a complementary upper bound.

Recovery can therefore come from other tasks that jointly cover the old task's directions. Weak coverage prevents rapid recovery when the task loss is not negligible relative to the remaining parameter error. The ratio $G_m/E_m$ records the current error state, so coverage alone cannot determine recovery speed across all learning histories. For a single-direction task $C_\tau=uu^\top$ with $\|u\|=1$, this ratio is the fraction of remaining error along $u$. When it stays bounded below, the required recovery time grows at least on the $1/h_\tau$ scale. What matters is whether subsequent tasks can correct the errors left by the current history.

\subsection{Attaining the recovery time scale}\label{sec:peak-duration}

The inverse-coverage dependence in Theorem~\ref{thm:distribution-recovery} can be attained. The following family gives the complete forgetting curve and realizes this recovery time scale through tasks that never repeat the evaluated task.

\begin{theorem}[Recovery at the inverse-coverage scale]\label{thm:peak-duration}
In two dimensions, let $C_\tau=e_1e_1^\top$. The background samples the single-input tasks
\[
x_\pm=(\sqrt h,\,\pm\sqrt{1-h})^\top,\qquad 0<h\le\tfrac12,
\]
with equal probability and labels $y_\pm=x_\pm^\top w^\star$. If $E_0>0$, then, with $\rho_h=1-2h(1-h)$,
\begin{equation}\label{eq:peak-duration-curve}
\frac{G_m}{E_0}=h(1-h)\rho_h^{m-1}\qquad(m\ge1).
\end{equation}
Consequently,
\begin{equation}\label{eq:peak-duration-summary}
\begin{aligned}
R_\varepsilon(1)&=\left\lceil\frac{\log(1/\varepsilon)}{-\log\rho_h}\right\rceil,\\
\frac{\max_mG_m}{E_0}&=h(1-h),\\
\frac{1}{E_0}\sum_{m\ge0}G_m&=\frac12.
\end{aligned}
\end{equation}
As $h\to0$, $R_\varepsilon(1)\sim\log(1/\varepsilon)/(2h)$,
while the normalized peak vanishes.
\end{theorem}

After fitting $\tau$, the remaining error lies along $e_2$. Substitution into~\eqref{eq:conditional-recovery-identity} gives a closed two-dimensional recursion, proved in Appendix~\ref{app:proof-peak-duration}. At the peak, $h_\tau=h$ and $G_1/E_1=1-h$, so recovery on the $1/h$ scale matches the order of the lower bound in Theorem~\ref{thm:distribution-recovery}. The background never repeats $\tau$: its two directions jointly provide the information needed for recovery.

In fact, the same recursion gives $E_m/E_0=h\rho_h^{m-1}$ for every $m\ge1$, so $G_m/E_m=1-h$ throughout the recovery phase. The small peak does not make the lower bound vanish: after the first background update, the old task still measures a fixed fraction of the remaining error. For fixed $\varepsilon$, the lower bound is of order $1/h$, and the exact recovery time has the same order.

The exact curve also records how interference changes with coverage. Reducing $h$ from $0.1$ to $0.01$ increases the time to reach one quarter of the peak from $7$ to $70$ updates, while lowering the normalized peak from $0.09$ to $0.0099$. The cumulative normalized loss remains $1/2$. In this family, weaker coverage reduces the initial disruption and slows its repair. Normalizing by $E_0$ separates these effects from changes in the preceding history.

\paragraph{Implications for continual learning.}
Recovery depends on the information supplied by subsequent learning. Revisiting the old task provides one source of this information; other tasks can jointly provide it by covering the directions on which performance depends. The resulting recovery can be slow when those directions receive little support, even though the old task is ultimately recoverable.

This mechanism suggests examining replay and training-data selection through the information they provide for correcting old-task errors. For broader models, this raises the question of which subsequent experiences can repair emerging forgetting. Evaluating stability requires considering both the disruption of learned performance and the opportunities for recovery supplied by future learning.

\section{Related Work}\label{sec:related-work}

Continual-learning methods address forgetting through replay \citep{lopez2017gradient,buzzega2020dark}, parameter regularization \citep{kirkpatrick2017overcoming,zenke2017synaptic,aljundi2018memory}, distillation \citep{li2017learning}, parameter isolation or expansion \citep{mallya2018packnet,yoon2018lifelong}, and Bayesian updates \citep{nguyen2018variational}. Surveys review these approaches \citep{delange2021continual,wang2024survey}, while \citet{benzing2022unifying} relates importance-based regularizers to Fisher information.

For jointly realizable linear regression, \citet{evron2022catastrophic} derive projection-based forgetting expressions and guarantees under different orders of a fixed task collection. \citet{jeong2023cyclic} provide dimension-independent cyclic bounds uniform over task collections. \citet{evron2026sgd} connect continual linear learning to SGD and improve bounds for i.i.d.\ uniform sampling from finite collections and random orderings without replacement. Related analyses connect weakly regularized classification to sequential max-margin problems \citep{evron2023classification}, characterize convergence and implicit bias with finite per-task gradient descent \citep{jung2025classification}, and study weakly regularized homogeneous networks through task constraint sets \citep{schliserman2026homogeneous}.

Statistical analyses derive forgetting and generalization expressions under isotropic Gaussian designs \citep{lin2023theory}, study orthogonally related tasks \citep{goldfarb2023orthogonal,goldfarb2025overparameterization}, and capture joint effects of similarity and overparameterization \citep{goldfarb2024joint}. SGD analyses characterize the roles of covariance, ordering, and step size \citep{ding2024understanding}. Teacher--student models relate input, feature, and target similarities to transfer and forgetting \citep{asanuma2021statistical,lee2021continual,hiratani2024tasksimilarity}; neural tangent kernel overlap describes further task interactions \citep{doan2021theoretical}. Within the jointly realizable, i.i.d., exact-fit setting, our identity retains task covariance and training--evaluation dependence to connect the task distribution to the magnitude, decay, and absence of historical forgetting.

Other work studies last-iterate convergence of incremental gradient and proximal methods \citep{cai2024lastiterate} and the effects of quadratic regularization on convergence and statistical performance \citep{zhao2024statistical,levinstein2025optimal,karpel2026l2}. Memory analyses address model capacity and buffer size \citep{banayeeanzade2024theoretical}, regularization through sketched Jacobians \citep{heckel2022provable}, and settings where more replay can increase forgetting \citep{mahdaviyeh2025replay}. Task presentation can be optimized through ordering and selection \citep{li2025taskorder,tsipory2025greedy,mori2025optimal}. Finite training-time analyses additionally address smoothly varying task paths \citep{xu2026trainingtime}.

Exact fitting connects to randomized row and block Kaczmarz methods \citep{strohmer2009randomized,needell2014paved,oswald2015convergence}. Second-moment analyses give exact reconstruction-error formulas \citep{agaskar2014randomized}; related work studies preasymptotic behavior \citep{jiao2017preasymptotic}, slow singular directions \citep{steinerberger2021small}, and Hilbert-space convergence \citep{guo2022rates}. Our framework uses the same type of operator to average updates while retaining the relation between historical evaluation and the trajectory. Proposition~\ref{prop:main-same-operator} shows that, even with the complete operator fixed, historical forgetting can differ: this relation contains information not determined by error second moments and independent-test risk.

Temporary forgetting and recovery have also been studied empirically. \citet{delange2023stability} show that performance drops after task switches can be followed by recovery, so evaluation only at task boundaries can miss intermediate losses. In the exact-fit linear setting, we study how subsequent tasks jointly support recovery, bound the required time, and construct a distribution attaining that timescale.


\bibliographystyle{icml2026}
\bibliography{references}

\clearpage
\appendix
\onecolumn

\section*{Contribution and Relation to Mathematical Tools}

Sequential exact fitting admits a projection recursion, and a linear operator describes the second moments of its random updates. These connections are established tools in continual learning and randomized Kaczmarz analysis \citep{evron2022catastrophic,agaskar2014randomized}. We use them to study how the task distribution determines historical forgetting. In a historical loss, the task's loss matrix appears alongside a parameter error that contains its update. Theorem~\ref{thm:exact-forgetting-identity} conditions on that task, averages the preceding and subsequent updates separately, and retains the joint contribution of its loss matrix and projection. This gives an exact expression for forgetting at every finite horizon.

Proposition~\ref{prop:main-same-operator} shows why the retained information matters. Two task distributions have the same complete second-moment operator and mean loss matrix, so their parameter-error second moments and independent-task risks agree at every step, yet their historical forgetting differs. Within the specified two-dimensional family of unit-norm, single-input tasks, the same pair attains the minimum and maximum forgetting at every horizon. Appendix~\ref{app:fixed-operator-range} gives the exact attainable range for a broader family of fixed operators. The identity isolates the third directional moments that still affect historical loss; moment constraints and finite-support constructions then determine the extrema.

Spectral decomposition and variational analysis separate the time dependence of forgetting from its task-dependent coefficients. They give decay guarantees, relate contraction to task coverage, and identify when slowly contracting error modes contribute to historical loss. For a prescribed task, the conditional identity gives its post-fit error second moment and its loss under subsequent updates. A bound on displacement under the mean projection update gives a lower bound on recovery time; an exact task family attains the inverse-coverage scale.

\paragraph{Guide to the appendix.}
Appendix~\ref{app:preliminaries} collects notation and preliminary lemmas. Appendices~\ref{app:exact-dynamics}, \ref{app:convergence}, and~\ref{app:recovery} follow Sections~\ref{sec:exact-dynamics}, \ref{sec:rates}, and~\ref{sec:recovery}, respectively, with complete proofs and related extensions. Appendix~\ref{app:reproducibility} reports the numerical results and their evaluation protocols.

\section{Preliminaries}\label{app:preliminaries}

We use the shared-solution, independent-sampling, and exact-fit assumptions of Section~\ref{sec:setup}. The following notation distinguishes historical averages from the loss of a particular learned task.

\begin{center}
\begin{tabularx}{\textwidth}{@{}lX@{}}
\toprule
Quantity & Meaning \\
\midrule
$P_t=I-X_t^+X_t$ & Projection onto the null space of task $t$. \\
$C_t=X_t^\top X_t/n_t$ & Matrix defining its squared loss. \\
$e_k=w_k-w^\star$, $X_\star=w^\star w^{\star\top}$ & Parameter error and its initial second moment for $w_0=0$. \\
$S_\Pi(A)=\mathbb E[PAP]$ & Evolution operator for error second moments. \\
$H_\Pi=\mathbb E[I-P]$, $V_\Pi=\Null(H_\Pi)^\perp$ & Mean coverage matrix and the subspace observed by the distribution. \\
$F^\Pi(k)$ & Expected loss averaged over all $k$ historical task occurrences. \\
$L^\Pi(t,m;w^\star)=\mathbb E\ell_t(w_{t+m})$ & Expected loss at age $m$ of the task sampled at position $t$. \\
$G_m=\mathbb E\ell_\tau(w_{T+1+m})$ & Loss of a prescribed task $\tau$, fitted after $T$ sampled tasks. \\
$E_m=\mathbb E\|w_{T+1+m}-w^\star\|^2$ & Remaining parameter error in the prescribed-task protocol. \\
\bottomrule
\end{tabularx}
\end{center}

The expectation in $L^\Pi$ includes the task at position $t$. In $G_m$, the task $\tau$ is fixed independently of the sampled history. Averaging $L^\Pi(t,k-t;w^\star)$ over $t=1,\ldots,k$ gives $F^\Pi(k)$; the term at $t=k$ is zero after exact fitting. The index $t$ refers to a task occurrence, so repeated draws of the same task contribute separately to the average.

\subsection{Exact-fit projection update}

\begin{lemma}[Exact-fit update under a common solution]\label{lem:appendix-exact-fit}
Let \((X,y)\) be a regression task with \(y=Xw^\star\), and let \(P:=I-X^+X\). If
\[
\mathcal S_{X,y}(u):=\arg\min_{w\in\R^d}\frac12\|w-u\|_2^2
\qquad \text{subject to}\qquad
Xw=y,
\]
then
\[
\mathcal S_{X,y}(u)=u+X^+(y-Xu)=w^\star+P(u-w^\star).
\]
\end{lemma}

\begin{proof}
The constrained minimization is equivalent to
\[
\min_{z\in\R^d}\frac12\|z\|_2^2
\qquad \text{subject to}\qquad
Xz=y-Xu,
\]
after substituting \(z=w-u\). The minimum-norm solution is \(z=X^+(y-Xu)\), so
\[
\mathcal S_{X,y}(u)=u+X^+(y-Xu).
\]
Under the common-solution assumption \(y=Xw^\star\), this becomes
\[
u+X^+X(w^\star-u)=u+(I-P)(w^\star-u)=w^\star+P(u-w^\star).
\]
\end{proof}

\subsection[Basic properties of the overlap operator]{Basic properties of \(S_\Pi\)}

\begin{lemma}[Self-adjointness and positivity of \(S_\Pi\)]\label{lem:appendix-SPi}
Let \(S_\Pi:\mathbb S^d\to\mathbb S^d\) be defined by \(S_\Pi(A)=\E[P_tAP_t]\). Then:
\begin{enumerate}[label=(\roman*),leftmargin=2.2em]
\item \(S_\Pi\) is self-adjoint with respect to the Frobenius inner product.
\item \(S_\Pi\) is positive semidefinite, that is,
\[
\langle A,S_\Pi(A)\rangle_F = \E\|P_tAP_t\|_F^2 \ge 0
\qquad \text{for all } A\in\mathbb S^d.
\]
\item \(S_\Pi\) preserves positive semidefiniteness: if \(A\succeq 0\), then \(S_\Pi(A)\succeq 0\).
\end{enumerate}
\end{lemma}

\begin{proof}
For \(A,B\in\mathbb S^d\),
\[
\langle A,S_\Pi(B)\rangle_F
=
\E\tr(AP_tBP_t)
=
\E\tr(P_tAP_tB)
=
\langle S_\Pi(A),B\rangle_F,
\]
so \(S_\Pi\) is self-adjoint. Also,
\[
\langle A,S_\Pi(A)\rangle_F
=
\E\tr(AP_tAP_t)
=
\E\tr(P_tAP_tP_tAP_t)
=
\E\|P_tAP_t\|_F^2 \ge 0,
\]
because \(P_t^2=P_t\) and \(P_tAP_t\) is symmetric. Finally, if \(A\succeq 0\), then \(P_tAP_t\succeq 0\) for every \(t\), hence its expectation \(S_\Pi(A)\) is also positive semidefinite.
\end{proof}

\subsection{Principal-angle identity}

\begin{lemma}[Principal-angle formula]\label{lem:appendix-principal-angles}
Let \(E\subset\R^d\) be an \(m\)-dimensional subspace, and let \(U\in\R^{d\times m}\) have orthonormal columns spanning \(E\). If \(P\) is the orthogonal projector onto another subspace \(N\subset\R^d\), then
\[
\|U^\top P U\|_F^2
=
\sum_{j=1}^m \cos^4(\theta_j(E,N)),
\]
where \(\theta_1(E,N),\dots,\theta_m(E,N)\) are the principal angles, padded with $\pi/2$ when $\dim N<m$, between \(E\) and \(N\).
\end{lemma}

\begin{proof}
The matrix \(U^\top P U\) is symmetric positive semidefinite. Its eigenvalues are the squared singular values of the restriction of \(P\) to \(E\), namely \(\cos^2(\theta_1(E,N)),\dots,\cos^2(\theta_m(E,N))\). Therefore
\[
\|U^\top P U\|_F^2
=
\tr\!\big((U^\top P U)^2\big)
=
\sum_{j=1}^m \cos^4(\theta_j(E,N)).
\]
\end{proof}

\section{Exact Dynamics and Fixed-Operator Comparisons}\label{app:exact-dynamics}\label{app:full-proofs}

This appendix follows Section~\ref{sec:exact-dynamics}. We first prove the exact identity and its spectral expansion, then give the construction behind Proposition~\ref{prop:main-same-operator}. The final subsection determines the attainable range of historical task loss for a family of fixed two-dimensional operators and proves the simultaneous extremality used in that proposition.

\subsection{Proof of Theorem~\ref{thm:exact-forgetting-identity}}\label{app:proof-exact-forgetting-identity}

\paragraph{Restatement.}
For every \(k\ge 2\),
\[
F^\Pi(k)=\frac1k\sum_{t=1}^{k-1}\mathbb E\tr\!\Big(S_\Pi^{k-t}(C_t)\,P_t\,S_\Pi^{t-1}(X_\star)\,P_t\Big).
\]

\begin{proof}
Let \(e_t:=w_t-w^\star\). By Lemma~\ref{lem:appendix-exact-fit},
\[
e_t=P_t e_{t-1},
\qquad
e_0=-w^\star.
\]
Hence
\[
e_t=-P_tP_{t-1}\cdots P_1w^\star.
\]

Fix \(t\le k-1\). Write
\[
R_{t,k}:=P_kP_{k-1}\cdots P_{t+1},
\qquad
L_{t-1}:=P_{t-1}P_{t-2}\cdots P_1,
\]
with \(L_0=I\). Then
\[
e_k=R_{t,k}e_t=-R_{t,k}P_tL_{t-1}w^\star,
\]
and therefore
\[
e_ke_k^\top
=
R_{t,k}P_tL_{t-1}X_\star L_{t-1}^\top P_tR_{t,k}^\top.
\]
Using \(\ell_t(w_k)=e_k^\top C_t e_k=\tr(C_t e_k e_k^\top)\) and cyclicity of trace, we obtain
\[
\ell_t(w_k)
=
\tr\!\big(R_{t,k}^\top C_t R_{t,k}\,P_t\,L_{t-1}X_\star L_{t-1}^\top P_t\big).
\]

Now condition on the current task. Because future tasks are independent of the present and past,
\[
\E\!\left[R_{t,k}^\top C_t R_{t,k}\mid X_t,y_t\right]
=
S_\Pi^{k-t}(C_t).
\]
Similarly, since the past tasks are i.i.d. and independent of the current task,
\[
\E[L_{t-1}X_\star L_{t-1}^\top]
=
S_\Pi^{t-1}(X_\star).
\]
Substituting these two identities and averaging over the current task yields
\[
\E\,\ell_t(w_k)
=
\E\tr\!\Big(S_\Pi^{k-t}(C_t)\,P_t\,S_\Pi^{t-1}(X_\star)\,P_t\Big).
\]

Finally,
\[
F^\Pi(k)=\E\!\left[\frac1k\sum_{t=1}^k \ell_t(w_k)\right].
\]
The last term vanishes because the \(k\)-th task is fit exactly, so \(\ell_k(w_k)=0\). Summing over \(t=1,\dots,k-1\) proves the claim.
\end{proof}

\subsection{Proof of Theorem~\ref{thm:exact-spectral-expansion}}\label{app:proof-exact-spectral-expansion}

\paragraph{Restatement.}
If \(\mathbb S^d=\mathcal E_1\oplus\cdots\oplus\mathcal E_R\) is the orthogonal eigenspace decomposition of \(S_\Pi\), with spectral levels \(\rho_1,\dots,\rho_R\) and Frobenius projectors \(\mathcal Q_1,\dots,\mathcal Q_R\), then
\[
F^\Pi(k)=\frac1k\sum_{r,s=1}^R \beta_{rs}\Gamma_{rs}(k),
\]
where
\[
\beta_{rs}:=\mathbb E\tr\!\big(\mathcal Q_r(C_t)\,P_t\,\mathcal Q_s(X_\star)\,P_t\big)
\]
and
\[
\Gamma_{rs}(k)=
\begin{cases}
(k-1)\rho_r^{\,k-1}, & r=s,\\[0.4em]
\rho_r\dfrac{\rho_r^{\,k-1}-\rho_s^{\,k-1}}{\rho_r-\rho_s}, & r\neq s.
\end{cases}
\]

\begin{proof}
Because \(S_\Pi\) is self-adjoint on \(\mathbb S^d\), the spectral theorem gives
\[
A=\sum_{r=1}^R \mathcal Q_r(A),
\qquad
S_\Pi^m(A)=\sum_{r=1}^R \rho_r^m \mathcal Q_r(A)
\]
for every \(A\in\mathbb S^d\) and \(m\ge 0\).

Apply this to the identity from Theorem~\ref{thm:exact-forgetting-identity}:
\[
S_\Pi^{k-t}(C_t)=\sum_{r=1}^R \rho_r^{k-t}\mathcal Q_r(C_t),
\qquad
S_\Pi^{t-1}(X_\star)=\sum_{s=1}^R \rho_s^{t-1}\mathcal Q_s(X_\star).
\]
Substituting these expansions into the trace and using bilinearity,
\begin{align*}
\mathbb E\tr\!\Big(S_\Pi^{k-t}(C_t)\,P_t\,S_\Pi^{t-1}(X_\star)\,P_t\Big)
&=
\sum_{r,s=1}^R
\rho_r^{k-t}\rho_s^{t-1}
\mathbb E\tr\!\big(\mathcal Q_r(C_t)\,P_t\,\mathcal Q_s(X_\star)\,P_t\big) \\
&=
\sum_{r,s=1}^R \rho_r^{k-t}\rho_s^{t-1}\beta_{rs}.
\end{align*}
Summing over \(t\) and dividing by \(k\) gives the claimed expansion.

It remains to evaluate \(\Gamma_{rs}(k):=\sum_{t=1}^{k-1}\rho_r^{k-t}\rho_s^{t-1}\). If \(r=s\), each term equals \(\rho_r^{k-1}\), so
\[
\Gamma_{rr}(k)=(k-1)\rho_r^{k-1}.
\]
If \(r\neq s\), this is a finite geometric sum:
\[
\Gamma_{rs}(k)
=
\rho_r\sum_{j=0}^{k-2}\rho_r^{k-2-j}\rho_s^j
=
\rho_r\frac{\rho_r^{k-1}-\rho_s^{k-1}}{\rho_r-\rho_s}.
\qedhere
\]
\end{proof}

\subsection{Identical second-moment dynamics and different forgetting}

For Proposition~\ref{prop:main-same-operator}, we construct two task laws that share the target, unit-norm inputs, mean loss matrix, and full second-moment operator. Corollary~\ref{cor:fixed-operator-extremizers} below proves that the same distributions are extremal at every horizon.

\begin{proposition}[The second-moment operator does not determine forgetting]\label{prop:appendix-same-operator}
There exist two realizable task laws \(\Pi_A,\Pi_B\) in dimension two, with the same common target \(w^\star=e_1\), zero initialization, and unit-norm one-row designs, such that
\[
\mathbb E_{\Pi_A}P=\mathbb E_{\Pi_B}P
=\mathbb E_{\Pi_A}C=\mathbb E_{\Pi_B}C=\tfrac12 I,
\qquad S_{\Pi_A}=S_{\Pi_B}.
\]
Their parameter-error second moments and independent-task risks coincide at every step, but
\[
F^{\Pi_B}(k)-F^{\Pi_A}(k)
=\frac{3}{32k}\sum_{m=1}^{k-1}
\left(\frac38\right)^{k-m-1}
\left[\left(\frac38\right)^m-\left(\frac18\right)^m\right]>0
\qquad(k\ge2).
\]
In particular, \(F^{\Pi_A}(2)=9/512\) and \(F^{\Pi_B}(2)=15/512\).
\end{proposition}

\begin{proof}
Under \(\Pi_A\), let the nullspace projector take the values
\[
P_0=\begin{pmatrix}1&0\\0&0\end{pmatrix},
\qquad
P_\pm=\frac18\begin{pmatrix}1&\pm\sqrt7\\\pm\sqrt7&7\end{pmatrix},
\]
with probabilities \(3/7,2/7,2/7\), respectively. Under \(\Pi_B\), use the complementary projectors \(I-P_0,I-P_+,I-P_-\) with the same probabilities. For each projector, choose a unit row \(x\) with \(x^\top x=I-P\) and set \(y=xw^\star\). Thus \(C=I-P\), \(\|C\|_F=1\), and all tasks share the stated target.

Direct averaging gives \(\mathbb E_{\Pi_A}P=I/2\). Consequently,
\[
S_{\Pi_B}(A)=\mathbb E_{\Pi_A}[(I-P)A(I-P)]
=A-\tfrac12 A-\tfrac12 A+S_{\Pi_A}(A)=S_{\Pi_A}(A).
\]
Writing this common operator as \(S\), its action is
\[
S\!\begin{pmatrix}a&b\\b&c\end{pmatrix}
=\begin{pmatrix}(7a+c)/16&b/8\\b/8&(a+7c)/16\end{pmatrix}.
\]
Its eigenvectors \(I\), \(Z=\operatorname{diag}(1,-1)\), and \(\left(\begin{smallmatrix}0&1\\1&0\end{smallmatrix}\right)\) have eigenvalues \(1/2\), \(\lambda=3/8\), and \(\nu=1/8\). Since \(e_1e_1^\top=(I+Z)/2\), both laws give
\[
\mathbb E[e_ke_k^\top]=S^k(e_1e_1^\top)
=\frac{2^{-k}}2 I+\frac{\lambda^k}2 Z.
\]
Their common independent-task risk is therefore
\(\mathbb E[e_k^\top\bar C e_k]=2^{-(k+1)}\).

To compute historical forgetting, define
\(M_m^A=\mathbb E_{\Pi_A}[P S^m(I-P)P]\), and analogously \(M_m^B\).
For a rank-one projector \(P=\left(\begin{smallmatrix}p&b\\b&1-p\end{smallmatrix}\right)\), put \(z=2p-1\), so \(b^2=p(1-p)\). The displayed formula for \(S\) yields
\[
P S^m(I-P)P
=\left(\frac{2^{-m}}2-\frac{\lambda^m z^2}2-2\nu^m b^2\right)P.
\]
For \(P_0\) this scalar is \((2^{-m}-\lambda^m)/2\); for \(P_\pm\) it is \(2^{-m}/2-(9\lambda^m+7\nu^m)/32\). Averaging gives
\[
M_m^{A,B}
=\left(\frac{2^{-m}}4-\frac{3\lambda^m+\nu^m}{16}\right)I
\mp\frac{3(\lambda^m-\nu^m)}{64}Z,
\]
where the minus sign corresponds to \(A\). The complement operation preserves the scalar above and reverses the \(Z\) coefficient, giving the formula for \(B\).
Theorem~\ref{thm:exact-forgetting-identity}, with \(m=k-t\), now gives
\[
F^{\Pi_j}(k)=\frac1k\sum_{m=1}^{k-1}
\langle M_m^j,S^{k-m-1}(e_1e_1^\top)\rangle_F,
\qquad j\in\{A,B\}.
\]
Subtracting and using \(\langle Z,Z\rangle_F=2\) proves the claimed difference, which is strictly positive because \(\lambda>\nu\). At \(m=1\), \(M_1^A=\operatorname{diag}(9,15)/256\) and \(M_1^B=\operatorname{diag}(15,9)/256\), giving the stated two-step values.
\end{proof}

This construction also permits a continuous change in forgetting with fixed second-moment dynamics. For \(\Pi_\theta=\theta\Pi_A+(1-\theta)\Pi_B\), \(0\le\theta\le1\), both \(S_{\Pi_\theta}\) and \(\bar C_{\Pi_\theta}\) are independent of \(\theta\). The exact identity is then linear in the remaining expectation over the evaluated task, so
\[
F^{\Pi_\theta}(k)=\theta F^{\Pi_A}(k)+(1-\theta)F^{\Pi_B}(k).
\]
Thus identical second-moment evolution and independent-task risk need not imply identical historical forgetting: the latter also retains the dependence between the evaluated task and the updates it participates in.

\subsection{Exact attainable range with a fixed operator}\label{app:fixed-operator-range}

Proposition~\ref{prop:appendix-same-operator} gives two task distributions with the same complete second-moment operator and different forgetting. We now solve the corresponding extremal problem in a two-dimensional family: with the operator, target, and evaluation times fixed, what values of old-task forgetting can a task distribution attain? The family uses unit-norm one-row designs, so $C=I-P$ and both $C$ and $P$ have rank one. Throughout this subsection, the target $w^\star\in\mathbb R^2$ is deterministic, $w_0=0$, and tasks are sampled i.i.d.

\paragraph{The fixed operator and its task distributions.}
Let
\[
Z_1=\begin{pmatrix}1&0\\0&-1\end{pmatrix},\qquad
Z_2=\begin{pmatrix}0&1\\1&0\end{pmatrix}.
\]
Every rank-one orthogonal projector in dimension two has the unique representation
\begin{equation}\label{eq:fixed-operator-circle}
P(z)=\tfrac12(I+z_1Z_1+z_2Z_2),\qquad z_1^2+z_2^2=1,
\qquad C(z)=I-P(z).
\end{equation}
For $1/2<s<1$, fix the operator $S_s$ by its action on the basis $I,Z_1,Z_2$ of $\mathbb S^2$:
\begin{equation}\label{eq:fixed-operator-levels}
S_s(I)=cI,\qquad S_s(Z_1)=aZ_1,\qquad S_s(Z_2)=bZ_2,
\qquad c=\tfrac12,\quad a=\tfrac{s}{2},\quad b=\tfrac{1-s}{2}.
\end{equation}
Write $\mathfrak P_s$ for all distributions of the projectors in~\eqref{eq:fixed-operator-circle} whose induced operator equals $S_s$. Equivalently, their circle coordinates satisfy
\begin{equation}\label{eq:fixed-operator-moments}
\mathbb Ez=0,\qquad
\mathbb E[zz^\top]=\begin{pmatrix}s&0\\0&1-s\end{pmatrix}.
\end{equation}
Indeed, $PZ_jP=z_jP$ gives
\[
S_\Pi(I)=\tfrac12(I+\mathbb Ez_1Z_1+\mathbb Ez_2Z_2),\qquad
S_\Pi(Z_j)=\tfrac12\bigl(\mathbb Ez_j I+\mathbb E[z_jz_1]Z_1+\mathbb E[z_jz_2]Z_2\bigr),
\]
which proves both directions of this equivalence. Every such projector law is realizable by choosing a unit row $x$ with $x^\top x=C(z)$ and setting $y=xw^\star$. In particular, all laws in $\mathfrak P_s$ have $\mathbb EP=\mathbb EC=I/2$ and exactly the same $S_s$.

\paragraph{The quantity being optimized.}
For an integer learning position $t\ge1$ and age $m\ge0$, define
\begin{equation}\label{eq:fixed-operator-age}
L^\Pi(t,m;w^\star):=\mathbb E_\Pi[\ell_t(w_{t+m})].
\end{equation}
The expectation follows the task learned at position $t$ through $m$ subsequent updates; the evaluated task itself is still sampled from $\Pi$. Write the fixed target matrix as
\[
W:=X_\star=w^\star w^{\star\top}=x_0I+x_1Z_1+x_2Z_2,
\qquad
x_0=\tfrac12\|w^\star\|^2,\quad
x_1=\tfrac12\bigl((w_1^\star)^2-(w_2^\star)^2\bigr),\quad
x_2=w_1^\star w_2^\star.
\]
Introduce
\begin{equation}\label{eq:fixed-operator-range-notation}
\begin{aligned}
R_s&=s(1-s),&
h_\Pi&=(\mathbb E_\Pi z_1^3,\mathbb E_\Pi z_1^2z_2),\\
v_t&=(a^{t-1}x_1,b^{t-1}x_2),&
D_m&=\tfrac12(a^m-b^m),\\
L_0(t,m;w^\star)&=\tfrac{x_0}{2}c^{t-1}\bigl[c^m-sa^m-(1-s)b^m\bigr].
\end{aligned}
\end{equation}
Here $D_m\ge0$, and all quantities except $h_\Pi$ are fixed once $s,w^\star,t,m$ are fixed.

\begin{theorem}[Exact attainable range with a fixed operator]\label{thm:fixed-operator-range}
For $1/2<s<1$, the set of third-moment vectors attained by $\mathfrak P_s$ is exactly
\begin{equation}\label{eq:fixed-operator-disk}
\{h_\Pi:\Pi\in\mathfrak P_s\}=\{h\in\mathbb R^2:\|h\|_2\le R_s\}.
\end{equation}
For every fixed deterministic target $w^\star$ and integers $t\ge1$, $m\ge0$,
\begin{equation}\label{eq:fixed-operator-loss}
L^\Pi(t,m;w^\star)=L_0(t,m;w^\star)-D_m\langle h_\Pi,v_t\rangle.
\end{equation}
Consequently, the exact range over all task distributions with this operator is
\begin{equation}\label{eq:fixed-operator-range}
\begin{aligned}
\min_{\Pi\in\mathfrak P_s}L^\Pi(t,m;w^\star)
&=L_0(t,m;w^\star)-R_sD_m\|v_t\|_2,\\
\max_{\Pi\in\mathfrak P_s}L^\Pi(t,m;w^\star)
&=L_0(t,m;w^\star)+R_sD_m\|v_t\|_2.
\end{aligned}
\end{equation}
Every value between the endpoints is attainable. Each endpoint can be attained using at most three task types. For a fixed $t$ and target, a single minimizing distribution and a single maximizing distribution attain the respective bounds for every age $m$; these distributions may depend on $t$ and the target.
\end{theorem}

\begin{proof}
\textbf{Step 1: the identity isolates the remaining third moments.}
The conditional loss identity proved in Theorem~\ref{thm:exact-forgetting-identity} gives
\begin{equation}\label{eq:fixed-operator-insertion}
L^\Pi(t,m;w^\star)
=\mathbb E_\Pi\tr\!\left[S_s^m(C)P S_s^{t-1}(W)P\right].
\end{equation}
Since $S_s$ is diagonal in the stated basis,
\[
B_t:=S_s^{t-1}(W)=c^{t-1}x_0I+a^{t-1}x_1Z_1+b^{t-1}x_2Z_2,
\qquad
S_s^m(C)=\tfrac12\bigl(c^mI-a^mz_1Z_1-b^mz_2Z_2\bigr).
\]
The rank-one relation $PB_tP=\tr(PB_t)P$ therefore reduces~\eqref{eq:fixed-operator-insertion} to
\begin{equation}\label{eq:fixed-operator-cubic}
L^\Pi(t,m;w^\star)
=\tfrac12\mathbb E_\Pi\!\left[
\bigl(c^{t-1}x_0+a^{t-1}x_1z_1+b^{t-1}x_2z_2\bigr)
\bigl(c^m-a^mz_1^2-b^mz_2^2\bigr)\right].
\end{equation}
The circle constraint and $\mathbb Ez=0$ imply
\[
\mathbb Ez_1z_2^2=-h_{\Pi,1},\qquad
\mathbb Ez_2^3=-h_{\Pi,2}.
\]
Expanding~\eqref{eq:fixed-operator-cubic} using these relations and~\eqref{eq:fixed-operator-moments} proves~\eqref{eq:fixed-operator-loss}. In particular, $m=0$ gives $L^\Pi(t,0;w^\star)=0$.

\medskip\noindent
\textbf{Step 2: every feasible third moment lies in the disk.}
Set $\zeta=z_1+\mathrm i z_2$ and $\eta=2s-1$. Then $|\zeta|=1$,
\[
\mathbb E\zeta=0,\qquad \mathbb E\zeta^2=\eta,\qquad
\mathbb E\zeta^3=4(h_{\Pi,1}+\mathrm i h_{\Pi,2}).
\]
For $U=\zeta^2-\eta$ and $V=\bar\zeta-\eta\zeta$, direct expansion yields
\[
\mathbb E|U|^2=\mathbb E|V|^2=1-\eta^2,
\qquad \mathbb E[U\bar V]=\mathbb E\zeta^3.
\]
Cauchy--Schwarz gives $|\mathbb E\zeta^3|\le1-\eta^2=4R_s$, proving $\|h_\Pi\|_2\le R_s$.

\medskip\noindent
\textbf{Step 3: every point of the disk is realizable.}
For any complex $\kappa$ with $|\kappa|=1$, put $q=\sqrt{1-\eta^2}$ and consider
\[
U_\kappa=
\begin{pmatrix}
0&\eta&\kappa q\\
1&0&0\\
0&q&-\kappa\eta
\end{pmatrix}.
\]
Its columns are orthonormal, so $U_\kappa$ is unitary. With $e=(1,0,0)^\top\in\mathbb C^3$,
\[
e^*U_\kappa e=0,\qquad
e^*U_\kappa^2e=\eta,\qquad
e^*U_\kappa^3e=\kappa(1-\eta^2).
\]
By the finite-dimensional spectral theorem, write $U_\kappa=\sum_j\zeta_j E_j$, where $|\zeta_j|=1$ and there are at most three distinct eigenvalues. Assign probability $p_j=e^*E_j e$ to $\zeta_j$. These weights are nonnegative and sum to one, and
\[
\mathbb E\zeta^r=\sum_jp_j\zeta_j^r=e^*U_\kappa^r e\qquad(r=1,2,3).
\]
Taking $z=(\operatorname{Re}\zeta,\operatorname{Im}\zeta)$ gives a law in $\mathfrak P_s$ with
\[
h_\Pi=R_s(\operatorname{Re}\kappa,\operatorname{Im}\kappa).
\]
Thus every boundary point is attained by at most three task types. Convex mixtures of these distributions preserve~\eqref{eq:fixed-operator-moments} and realize every interior point, proving~\eqref{eq:fixed-operator-disk}.

\medskip\noindent
\textbf{Step 4: optimize the exact loss over the disk.}
For $v_t\ne0$ and $m\ge1$,~\eqref{eq:fixed-operator-loss} is minimized at $h=R_sv_t/\|v_t\|_2$ and maximized at its negative. These are realizable boundary points by Step 3. Every intermediate value follows by convex mixing, since~\eqref{eq:fixed-operator-loss} is affine in $h$. The optimizing directions depend on $t$ and the target but not on $m$, because $D_m>0$ for every $m\ge1$. If $m=0$ or $v_t=0$, the interval collapses and any of the boundary distributions attains it. This proves all claims.
\end{proof}

The third-moment term in~\eqref{eq:fixed-operator-cubic} arises because the direction entering the evaluation factor is the same direction entering $PB_tP$. The complete operator $S_s$ fixes only the first two circle moments. Thus~\eqref{eq:fixed-operator-range} quantifies exactly how much old-task forgetting can vary while parameter-error second moments and independent-task risk remain fixed. For example, the latter risk is $\langle I/2,S_s^k(W)\rangle=\|w^\star\|^2 2^{-(k+1)}$ for every law in $\mathfrak P_s$.

\begin{corollary}[Simultaneous extremizers for a coordinate target]\label{cor:fixed-operator-extremizers}
Fix $w^\star=e_1$. Let $\Pi_+$ assign the circle points
\[
(1,0),\qquad(-s,\sqrt{1-s^2}),\qquad(-s,-\sqrt{1-s^2})
\]
probabilities $s/(1+s)$, $1/[2(1+s)]$, and $1/[2(1+s)]$, respectively. Let $\Pi_-$ replace each projector $P$ by $I-P$. Then for every $\Pi\in\mathfrak P_s$, all $t\ge1$, $m\ge0$, and $k\ge1$,
\[
L^{\Pi_+}(t,m;e_1)\le L^\Pi(t,m;e_1)\le L^{\Pi_-}(t,m;e_1),
\qquad
F^{\Pi_+}(k)\le F^\Pi(k)\le F^{\Pi_-}(k).
\]
Thus the same two distributions also minimize and maximize the peak historical forgetting over $\mathfrak P_s$.
\end{corollary}

\begin{proof}
Direct averaging gives~\eqref{eq:fixed-operator-moments} and $h_{\Pi_+}=(R_s,0)$. Complementing a projector replaces $z$ by $-z$, preserving $S_s$ and giving $h_{\Pi_-}=(-R_s,0)$. For $w^\star=e_1$, $(x_0,x_1,x_2)=(1/2,1/2,0)$, so $v_t=(a^{t-1}/2,0)$ always points in the same direction. Theorem~\ref{thm:fixed-operator-range} proves the age inequalities. Averaging them through $F^\Pi(k)=k^{-1}\sum_{t=1}^{k-1}L^\Pi(t,k-t;e_1)$ proves the historical inequalities, and taking the supremum over $k$ proves the peak statement.
\end{proof}

For $s=3/4$, $\Pi_+$ and $\Pi_-$ coincide with $\Pi_A$ and $\Pi_B$ in Proposition~\ref{prop:appendix-same-operator}. Hence that construction attains both ends of the exact interval $9/512\le F^\Pi(2)\le15/512$ over $\mathfrak P_{3/4}$, and the same pair is extremal at every later horizon. This proves the extremal claims of Proposition~\ref{prop:main-same-operator}.

\section{Convergence and Task Geometry}\label{app:convergence}\label{app:additional-results}

This appendix gives the proofs for Section~\ref{sec:rates}. We begin with the asymptotic expansion and its coefficient, then establish the exponential bound, its restriction to observable directions, the variational characterization, and the zero-forgetting criterion. The final two subsections relate these results to projection-residual bounds and explain the role of loss scale.

\subsection{Proof of Theorem~\ref{thm:leading-term}}\label{app:proof-leading-term}

\paragraph{Restatement.}
Assume \(\eta_\Pi<\rho_\Pi\). Then, as \(k\to\infty\),
\[
F^\Pi(k)=c_\Pi^{\mathrm{top}}(w^\star)\rho_\Pi^{\,k-1}
+
O\!\left(\frac{\rho_\Pi^{\,k-1}}{k}+\eta_\Pi^{\,k-1}\right).
\]

\begin{proof}
Apply Theorem~\ref{thm:exact-spectral-expansion}. Order the spectral levels as
\[
\rho_1=\rho_\Pi>\rho_2\ge\cdots\ge \rho_R\ge 0,
\]
and write \(\mathcal Q_1=\mathcal Q_\Pi\). Then \(\max_{r\ge 2}\rho_r=\eta_\Pi\), and
\[
F^\Pi(k)=\frac1k\sum_{r,s=1}^R \beta_{rs}\Gamma_{rs}(k).
\]

The top-top term is
\[
\frac1k\beta_{11}\Gamma_{11}(k)
=
\frac1k\beta_{11}(k-1)\rho_\Pi^{k-1}
=
\beta_{11}\rho_\Pi^{k-1}+O\!\left(\frac{\rho_\Pi^{k-1}}{k}\right).
\]
Since \(\mathcal Q_1=\mathcal Q_\Pi\), we have \(\beta_{11}=c_\Pi^{\mathrm{top}}(w^\star)\).

If \(s\ge 2\), then
\[
\Gamma_{1s}(k)=\rho_\Pi\frac{\rho_\Pi^{k-1}-\rho_s^{k-1}}{\rho_\Pi-\rho_s}=O(\rho_\Pi^{k-1}),
\]
because \(\rho_s\le \eta_\Pi<\rho_\Pi\). Multiplying by \(1/k\) gives \(O(\rho_\Pi^{k-1}/k)\). The same estimate holds for \(\Gamma_{r1}(k)\) with \(r\ge 2\).

If \(r,s\ge 2\), then \(\rho_r,\rho_s\le \eta_\Pi\), so
\[
\Gamma_{rs}(k)=O(k\eta_\Pi^{k-1}),
\]
and the prefactor \(1/k\) reduces this to \(O(\eta_\Pi^{k-1})\).

Combining the top-top, mixed, and lower-lower contributions yields
\[
F^\Pi(k)=c_\Pi^{\mathrm{top}}(w^\star)\rho_\Pi^{k-1}
+
O\!\left(\frac{\rho_\Pi^{k-1}}{k}+\eta_\Pi^{k-1}\right).
\]
\end{proof}

\subsection{Nonnegativity and generic positivity of the top coefficient}\label{app:proof-top-coefficient}

\begin{proposition}[Nonnegativity and generic positivity of the top coefficient]\label{prop:appendix-top-coefficient}
Assume \(\eta_\Pi<\rho_\Pi\). For each \(u\in\R^d\), let
\[
X_u:=uu^\top,
\qquad
c_\Pi^{\mathrm{top}}(u):=
\E\tr\!\big(\mathcal Q_\Pi(C_t)\,P_t\,\mathcal Q_\Pi(X_u)\,P_t\big).
\]
Then \(c_\Pi^{\mathrm{top}}(u)\ge 0\) for every \(u\in\R^d\). Moreover, there exists a symmetric positive-semidefinite matrix \(M_\Pi^{\mathrm{top}}\in\mathbb S^d\) such that
\[
c_\Pi^{\mathrm{top}}(u)=u^\top M_\Pi^{\mathrm{top}}u
\qquad\text{for every }u\in\R^d.
\]
Consequently, either \(c_\Pi^{\mathrm{top}}(\cdot)\equiv 0\), or its zero set is a proper quadratic variety and therefore has Lebesgue measure zero.
\end{proposition}

\begin{proof}
Fix \(u\in\R^d\). Keep the \(X\)-marginal induced by \(\Pi\) (equivalently, the law of \(P_t\) and \(C_t\)), and define a new realizable task law \(\Pi^{(u)}\) by drawing \(X_t\) from that marginal and setting \(y_t:=X_tu\). Then \(S_{\Pi^{(u)}}=S_\Pi\), so \(\rho_{\Pi^{(u)}}=\rho_\Pi\), \(\eta_{\Pi^{(u)}}=\eta_\Pi\), and \(\mathcal Q_{\Pi^{(u)}}=\mathcal Q_\Pi\). Writing \(F_u^\Pi(k)\) for the corresponding forgetting quantity under \(\Pi^{(u)}\), Theorem~\ref{thm:leading-term} applied to this modified law gives
\[
F_u^\Pi(k)=c_\Pi^{\mathrm{top}}(u)\rho_\Pi^{k-1}
+
O\!\left(\frac{\rho_\Pi^{k-1}}{k}+\eta_\Pi^{k-1}\right),
\]
where \(F_u^\Pi(k)\ge 0\) is the corresponding forgetting quantity. Dividing by \(\rho_\Pi^{k-1}\) and using \(\eta_\Pi<\rho_\Pi\), we obtain
\[
\frac{F_u^\Pi(k)}{\rho_\Pi^{k-1}}
=
c_\Pi^{\mathrm{top}}(u)
+
O\!\left(\frac1k+\Bigl(\frac{\eta_\Pi}{\rho_\Pi}\Bigr)^{k-1}\right).
\]
Taking \(k\to\infty\) yields \(c_\Pi^{\mathrm{top}}(u)\ge 0\).

Now \(u\mapsto X_u=uu^\top\) is quadratic in \(u\), and the map \(X\mapsto \E\tr(\mathcal Q_\Pi(C_t)\,P_t\,\mathcal Q_\Pi(X)\,P_t)\) is linear in \(X\). Therefore \(u\mapsto c_\Pi^{\mathrm{top}}(u)\) is a quadratic form. Hence there exists a symmetric matrix \(M_\Pi^{\mathrm{top}}\) such that
\[
c_\Pi^{\mathrm{top}}(u)=u^\top M_\Pi^{\mathrm{top}}u.
\]
Since this quantity is nonnegative for every \(u\), the matrix \(M_\Pi^{\mathrm{top}}\) is positive semidefinite.

If \(M_\Pi^{\mathrm{top}}=0\), then \(c_\Pi^{\mathrm{top}}(\cdot)\equiv 0\). Otherwise \(M_\Pi^{\mathrm{top}}\) has at least one positive eigenvalue, so the set \(\{u:c_\Pi^{\mathrm{top}}(u)=0\}\) is the null set of a nonzero quadratic polynomial, hence a proper quadratic variety of Lebesgue measure zero.
\end{proof}

\subsection{Proof of Theorem~\ref{thm:upper-bound}}\label{app:proof-upper-bound}

\paragraph{Restatement.}
For every \(k\ge 2\),
\[
F^\Pi(k)\le \frac{k-1}{k}\rho_\Pi^{k-1}\|X_\star\|_F\,\E\|C_t\|_F.
\]

\begin{proof}
Starting from Theorem~\ref{thm:exact-forgetting-identity},
\[
F^\Pi(k)=\frac1k\sum_{t=1}^{k-1}\mathbb E\tr\!\Big(S_\Pi^{k-t}(C_t)\,P_t\,S_\Pi^{t-1}(X_\star)\,P_t\Big).
\]
For symmetric matrices \(A,B\) and an orthogonal projector \(P\),
\[
|\tr(APBP)|
\le
\|A\|_F\,\|PBP\|_F
\le
\|A\|_F\,\|B\|_F.
\]
Therefore
\[
F^\Pi(k)
\le
\frac1k\sum_{t=1}^{k-1}\E\!\left[\|S_\Pi^{k-t}(C_t)\|_F\,\|S_\Pi^{t-1}(X_\star)\|_F\right].
\]
Since \(\|S_\Pi\|_{\mathrm{op},F}=\rho_\Pi\),
\[
\|S_\Pi^{k-t}(C_t)\|_F\le \rho_\Pi^{k-t}\|C_t\|_F,
\qquad
\|S_\Pi^{t-1}(X_\star)\|_F\le \rho_\Pi^{t-1}\|X_\star\|_F.
\]
Multiplying these estimates gives
\[
\|S_\Pi^{k-t}(C_t)\|_F\,\|S_\Pi^{t-1}(X_\star)\|_F
\le
\rho_\Pi^{k-1}\|C_t\|_F\,\|X_\star\|_F.
\]
Hence
\[
F^\Pi(k)\le \frac1k\sum_{t=1}^{k-1}\rho_\Pi^{k-1}\|X_\star\|_F\,\E\|C_t\|_F
=
\frac{k-1}{k}\rho_\Pi^{k-1}\|X_\star\|_F\,\E\|C_t\|_F.
\]
\end{proof}

\subsection{Proof of Corollary~\ref{cor:observable-decay}}\label{app:proof-observable-decay}

\begin{proof}
Write \(H=H_\Pi\), \(N=N_\Pi\), \(V=V_\Pi\), and \(R=R_\Pi\). Since \(I-P_t\) is an orthogonal projector, for every deterministic \(v\in\mathbb R^d\),
\[
v^\top Hv=\mathbb E[v^\top(I-P_t)v]=\mathbb E\|(I-P_t)v\|_2^2.
\]
For each vector in a finite orthonormal basis of \(N\), this expectation is zero. Intersecting these probability-one events shows that \(P_t|_N=I_N\) almost surely. Because \(P_t\) is symmetric, it also preserves \(V=N^\perp\). Moreover, \(C_tP_t=0\), so \(C_t\) annihilates \(N\) and, by symmetry, \(C_t=RC_tR\). These statements hold simultaneously along the countable task sequence with probability one.

If \(V=\{0\}\), then \(C_t=0\) almost surely and every loss is zero. Suppose henceforth that \(V\ne\{0\}\). The restriction \(H|_V\) is positive definite: a vector with zero quadratic form would belong to both \(\ker H=N\) and \(V\). Finite dimensionality therefore gives \(g_\Pi=\lambda_{\min}(H|_V)>0\). Also \(0\preceq H\preceq I\), so \(g_\Pi\le1\).

Let \(e_t=w_t-w^\star\), \(\widetilde e_t=Re_t\), and \(\widetilde P_t=P_t|_V\). The restricted dynamics and losses are
\[
\widetilde e_t=\widetilde P_t\widetilde e_{t-1},
\qquad \widetilde e_0=-Rw^\star,
\qquad e_k^\top C_t e_k=\widetilde e_k^\top(C_t|_V)\widetilde e_k.
\]
Thus the historical forgetting is unchanged by this restriction. The restricted projectors remain i.i.d., and \(\|C_t|_V\|_F=\|C_t\|_F\).

The operator \(S_{\Pi,V}\) is self-adjoint and positive semidefinite on the symmetric matrices over \(V\). For any such matrix \(A\),
\[
\begin{aligned}
\langle A,S_{\Pi,V}(A)\rangle_F
&=\mathbb E\|\widetilde P_t A\widetilde P_t\|_F^2
\le\mathbb E\|A\widetilde P_t\|_F^2\\
&=\tr\!\big(A^2\mathbb E[\widetilde P_t]\big)
\le(1-g_\Pi)\|A\|_F^2,
\end{aligned}
\]
where \(\mathbb E[\widetilde P_t]=I_V-H|_V\). The Rayleigh variational principle implies \(\rho_{\Pi,V}\le1-g_\Pi<1\).

Finally, the exact identity and the norm estimate in Theorem~\ref{thm:upper-bound} hold for these restricted orthogonal projectors, including tasks for which \(\widetilde P_t=0\). Using the initial matrix \((Rw^\star)(Rw^\star)^\top\) gives
\[
F^\Pi(k)\le\frac{k-1}{k}\rho_{\Pi,V}^{k-1}\|Rw^\star\|_2^2\,\mathbb E\|C_t\|_F
\le K_\Pi(1-g_\Pi)^{k-1},\qquad k\ge2.
\]
This also covers \(\rho_{\Pi,V}=0\), when the upper bound is zero.
\end{proof}

\subsection{Proof of Proposition~\ref{prop:variational-rho}}\label{app:proof-variational-rho}

\paragraph{Restatement.}
For every nonzero \(A\in\mathbb S^d\),
\[
0\le \mathcal I_\Pi(A)\le 1,
\qquad
\rho_\Pi=\max_{A\in\mathbb S^d,\;A\neq 0}\mathcal I_\Pi(A).
\]
Moreover, for \(A=UU^\top\),
\[
\mathcal I_\Pi(UU^\top)=\frac{\E\|U^\top P_tU\|_F^2}{\|UU^\top\|_F^2},
\]
and for an \(m\)-dimensional subspace \(E\),
\[
\mathcal I_\Pi(\Pi_E)=\frac1m\E\sum_{j=1}^m \cos^4(\theta_j(E,\operatorname{null}(X_t))).
\]

\begin{proof}
The inequalities \(0\le \mathcal I_\Pi(A)\le 1\) are immediate from
\[
\mathcal I_\Pi(A)=\frac{\E\|P_tAP_t\|_F^2}{\|A\|_F^2}
\]
and the projector inequality \(\|P_tAP_t\|_F\le \|A\|_F\).

By Lemma~\ref{lem:appendix-SPi},
\[
\langle A,S_\Pi(A)\rangle_F = \E\|P_tAP_t\|_F^2,
\]
so
\[
\mathcal I_\Pi(A)=\frac{\langle A,S_\Pi(A)\rangle_F}{\|A\|_F^2}.
\]
Since \(S_\Pi\) is self-adjoint and positive semidefinite, the Rayleigh--Ritz variational principle gives
\[
\max_{A\neq 0}\mathcal I_\Pi(A)=\|S_\Pi\|_{\mathrm{op},F}=\rho_\Pi.
\]

If \(A=UU^\top\), then
\[
P_tAP_t=(P_tU)(P_tU)^\top,
\]
hence
\[
\|P_tAP_t\|_F^2
=
\tr\!\big((P_tUU^\top P_t)^2\big)
=
\tr\!\big((U^\top P_tU)^2\big)
=
\|U^\top P_tU\|_F^2.
\]
This proves the second identity.

Finally, if \(U\) has orthonormal columns spanning \(E\), then \(\Pi_E=UU^\top\) and \(\|\Pi_E\|_F^2=\tr(\Pi_E)=m\). Therefore
\[
\mathcal I_\Pi(\Pi_E)=\frac{1}{m}\E\|U^\top P_tU\|_F^2.
\]
Lemma~\ref{lem:appendix-principal-angles} identifies the right-hand side with
\[
\frac1m\E\sum_{j=1}^m \cos^4(\theta_j(E,\operatorname{null}(X_t))),
\]
which is the claimed geometric formula.
\end{proof}

\subsection{Proof of Proposition~\ref{prop:commuting-zero-forgetting}}\label{app:proof-commuting-zero-forgetting}

\paragraph{Restatement.}
Fix the marginal distribution of \(X\) under the assumptions of Section~\ref{sec:setup}, and define \(\Pi_u\) by setting \(y=Xu\), with \(w_0=0\). Pairwise almost-sure commutativity of independent task projectors is equivalent both to \(F^{\Pi_u}(2)=0\) for every \(u\in\mathbb R^d\) and to \(F^{\Pi_u}(k)=0\) for every \(u\in\mathbb R^d\) and every \(k\ge 1\).

\begin{proof}
We prove (i)\(\Rightarrow\)(iii)\(\Rightarrow\)(ii)\(\Rightarrow\)(i), with the conditions numbered as in the main text.

Assume (i), and fix \(u\in\mathbb R^d\) and \(k\ge 1\). Since the tasks are i.i.d., every pair among \(P_1,\dots,P_k\) commutes with probability one; intersecting these finitely many events gives simultaneous pairwise commutativity almost surely. Let \(e_t=w_t-u\). The exact-fit dynamics give
\[
e_k=-P_kP_{k-1}\cdots P_1u.
\]
For every \(t\le k\), moving \(P_t\) to the front yields
\[
e_k=P_t\big(-P_k\cdots P_{t+1}P_{t-1}\cdots P_1u\big).
\]
Hence \(e_k\in\Range(P_t)=\Null(X_t)\), and therefore
\[
\ell_t(w_k)=\frac1{n_t}\|X_te_k\|_2^2=0.
\]
Averaging over \(t\) and taking expectations proves (iii). Taking \(k=2\) proves (iii)\(\Rightarrow\)(ii).

Now assume (ii). Since \(e_2=-P_2P_1u\) and the second task is fit exactly,
\[
F^{\Pi_u}(2)=\frac12\mathbb E\!\left[\frac1{n_1}\|X_1P_2P_1u\|_2^2\right].
\]
Let \(b_1,\dots,b_d\) be the standard basis of \(\mathbb R^d\). For each \(j\), condition (ii) with \(u=b_j\), together with nonnegativity of the integrand and \(n_1>0\), implies \(X_1P_2P_1b_j=0\) almost surely. The marginal law of \((X_1,X_2)\) is the same for every target, so intersecting these \(d\) probability-one events yields
\[
X_1P_2P_1=0\qquad\text{almost surely}.
\]
Multiplying on the left by \(X_1^+\), and using \(X_1^+X_1=I-P_1\), gives
\[
(I-P_1)P_2P_1=0,
\qquad P_2P_1=P_1P_2P_1.
\]
Since both projectors are symmetric, transposing the latter identity gives
\[
P_1P_2=P_1P_2P_1=P_2P_1.
\]
Thus independent task projectors commute almost surely, proving (i).
\end{proof}

\subsection{Comparison with projection-residual bounds}\label{app:projection-bounds}

Throughout this subsection, assume in addition that \(\|X_t\|_2\le 1\) almost surely and \(\|w^\star\|_2\le 1\). Define the residual surrogate
\[
\mathcal R^\Pi(k):=
\E\!\left[\frac1k\sum_{t=1}^{k-1}\|(\Id-P_t)P_k\cdots P_1\|_2^2\right].
\]
Let
\[
\Pbar:=\E[P_t],
\qquad
f_m:=\E\|(\Id-P)P_m\cdots P_1\|_F^2,
\]
where \(P,P_1,\dots,P_m\) are i.i.d. projector samples from \(\Pi\), and an empty product is $I$.

\begin{proposition}[Residual surrogate upper-bounds forgetting]\label{prop:appendix-residual-surrogate}
Under the additional normalization \(\|X_t\|_2\le 1\) and \(\|w^\star\|_2\le 1\),
\[
F^\Pi(k)\le \mathcal R^\Pi(k).
\]
\end{proposition}

\begin{proof}
For each \(t\le k\),
\[
\ell_t(w_k)=\frac1{n_t}\|X_t(w_k-w^\star)\|_2^2\le \|X_t\|_2^2\,\|w_k-w^\star\|_2^2.
\]
More sharply, since \(P_t=I-X_t^+X_t\),
\[
\|X_t(w_k-w^\star)\|_2^2 \le \|(\Id-P_t)(w_k-w^\star)\|_2^2.
\]
Now \(w_k-w^\star=-P_k\cdots P_1w^\star\), and \(\|w^\star\|_2\le 1\), so
\[
\ell_t(w_k)\le \|(\Id-P_t)P_k\cdots P_1\|_2^2.
\]
Averaging over \(t\) and expectations yields the claim.
\end{proof}

\begin{proposition}[Exact Frobenius-chain formula]\label{prop:appendix-frobenius-chain}
For every \(m\ge 0\),
\[
f_m=\tr\!\big(S_\Pi^m(\Id-\Pbar)\big).
\]
Equivalently, if
\[
\kappa_0:=d,\qquad \kappa_1:=\tr(\Pbar),
\qquad
\kappa_m:=\tr\!\big(S_\Pi^{m-1}(\Pbar)\big)\quad (m\ge 2),
\]
then
\[
f_m=\kappa_m-\kappa_{m+1}\ge 0.
\]
Moreover, the sequence \(\{f_m\}_{m\ge 0}\) is non-increasing.
\end{proposition}

\begin{proof}
Using symmetry and idempotence of the projectors,
\[
f_m
=
\E\tr\!\big(P_1\cdots P_m(\Id-P)P_m\cdots P_1\big).
\]
First average over the independent copy \(P\), replacing \(\Id-P\) by \(\Id-\Pbar\). Then average successively over \(P_m,\dots,P_1\). Each step applies the map \(A\mapsto \E[P_tAP_t]=S_\Pi(A)\), hence
\[
f_m=\tr\!\big(S_\Pi^m(\Id-\Pbar)\big).
\]
Since $S_\Pi(I)=\Pbar$, we have $\kappa_m=\tr(S_\Pi^m(I))$ for every $m\ge0$. Linearity therefore gives
\[
f_m=\tr\!\big(S_\Pi^m(I)-S_\Pi^{m+1}(I)\big)=\kappa_m-\kappa_{m+1}.
\]
Finally, if \(A_m:=S_\Pi^m(\Id-\Pbar)\), then \(A_m\succeq 0\) by Lemma~\ref{lem:appendix-SPi}, and
\[
f_{m+1}=\tr(S_\Pi(A_m))=\tr(\Pbar A_m)\le \tr(A_m)=f_m,
\]
because \(0\preceq \Pbar\preceq \Id\).
\end{proof}

\begin{proposition}[Projection-based transfer bound]\label{prop:appendix-transfer}
For every \(k\ge 2\), letting \(h:=\lceil (k-1)/2\rceil\),
\[
\mathcal R^\Pi(k)\le 4f_h.
\]
Consequently, under the additional normalization assumptions,
\[
F^\Pi(k)\le 4f_h.
\]
\end{proposition}

\begin{proof}
For \(1\le t\le k-1\), define
\[
r_{k,t}:=\E\|(\Id-P_t)P_k\cdots P_1\|_2^2.
\]
We spell out the decomposition used in the random-order argument of \citet{evron2022catastrophic}. Put $Q=I-P_t$, $R=P_k\cdots P_{t+1}$, and $L=P_{t-1}\cdots P_1$, with $L=I$ when $t=1$. Then
\[
QRP_tL=QRL-QRQL,\qquad
\|QRQL\|_2\le\min\{\|QR\|_2,\|QL\|_2\}.
\]
The second inequality uses the fact that products of orthogonal projectors have norm at most one. Define
\[
a_m:=\E\|(\Id-P)P_m\cdots P_1\|_2^2.
\]
The matrix $Q$ is independent of both $R$ and $L$, and $RL$ is a product of $k-1$ independent projectors. Squaring the decomposition and taking expectations gives
\[
r_{k,t}\le 2a_{k-1}+2\min\{a_{k-t},a_{t-1}\}
\le 2f_{k-1}+2f_{\max\{k-t,t-1\}}.
\]
Since $\max\{k-t,t-1\}\ge h$ and $f_m$ is non-increasing, each term is at most $4f_h$. Averaging over $t$ gives
\[
\mathcal R^\Pi(k)\le 4f_h.
\]
The bound for \(F^\Pi(k)\) then follows from Proposition~\ref{prop:appendix-residual-surrogate}.
\end{proof}

\begin{corollary}[One-step forgetting and pairwise overlap]\label{cor:appendix-f1}
The one-step Frobenius surrogate satisfies
\[
f_1=\tr(\Pbar)-\tr(\Pbar^2).
\]
If \(\lambda_1,\dots,\lambda_d\in[0,1]\) are the eigenvalues of \(\Pbar\), then
\[
f_1=\sum_{i=1}^d \lambda_i(1-\lambda_i).
\]
\end{corollary}

\begin{proof}
Apply Proposition~\ref{prop:appendix-frobenius-chain} with \(m=1\):
\[
f_1=\kappa_1-\kappa_2=\tr(\Pbar)-\tr(\Pbar^2).
\]
Diagonalizing \(\Pbar\) yields the eigenvalue formula.
\end{proof}

\begin{corollary}[Exponential decay from spectral diffuseness]\label{cor:appendix-exp}
Let \(\lambda_\Pi:=\|\Pbar\|_2\). Then
\[
f_m\le \lambda_\Pi^{m-1}f_1
\qquad \text{for every } m\ge 1.
\]
Consequently, for \(h=\lceil (k-1)/2\rceil\),
\[
F^\Pi(k)\le \mathcal R^\Pi(k)\le 4\lambda_\Pi^{h-1}f_1.
\]
\end{corollary}

\begin{proof}
Let \(A_m:=S_\Pi^m(\Id-\Pbar)\succeq 0\). Then
\[
f_{m+1}=\tr(S_\Pi(A_m))=\tr(\Pbar A_m)\le \|\Pbar\|_2\,\tr(A_m)=\lambda_\Pi f_m.
\]
Iterating proves the first claim, and Proposition~\ref{prop:appendix-transfer} yields the forgetting bound.
\end{proof}

\begin{corollary}[Recovering the coarse \(O(1/k)\) rate]\label{cor:appendix-coarse}
Let \(\nu_\Pi:=\tr(\Pbar)=\E\,\tr(P_t)\). Then
\[
f_m\le \frac{\nu_\Pi}{m}
\qquad \text{for every } m\ge 1,
\]
and therefore
\[
F^\Pi(k)\le \mathcal R^\Pi(k)\le \frac{8\nu_\Pi}{k-1}.
\]
\end{corollary}

\begin{proof}
Since \(f_m=\kappa_m-\kappa_{m+1}\) and \(f_m\) is non-increasing,
\[
mf_m\le \sum_{s=1}^m f_s=\kappa_1-\kappa_{m+1}\le \kappa_1=\nu_\Pi.
\]
Thus \(f_m\le \nu_\Pi/m\). Combining this with Proposition~\ref{prop:appendix-transfer} gives
\[
F^\Pi(k)\le \mathcal R^\Pi(k)\le 4f_h\le \frac{4\nu_\Pi}{h}\le \frac{8\nu_\Pi}{k-1}.
\]
\end{proof}

\subsection{Loss scale under unchanged projection updates}

\begin{proposition}[Projector statistics are blind to loss scale]\label{prop:appendix-scale-blind}
Fix any realizable task law \(\Pi\), and for \(\varepsilon>0\) define a rescaled law \(\Pi^{(\varepsilon)}\) by
\[
(X_t,y_t)\mapsto (\varepsilon X_t,\varepsilon y_t).
\]
Then:
\begin{enumerate}[label=(\roman*),leftmargin=2.2em]
\item the exact-fit projector law is unchanged, so \(P_t\), \(\Pbar\), \(S_\Pi\), \(\rho_\Pi\), and the projection-residual quantities in Appendix~\ref{app:projection-bounds} are the same under \(\Pi\) and \(\Pi^{(\varepsilon)}\);
\item the actual forgetting loss scales by \(\varepsilon^2\), namely
\[
F^{\Pi^{(\varepsilon)}}(k)=\varepsilon^2 F^\Pi(k).
\]
\end{enumerate}
\end{proposition}

\begin{proof}
The Moore--Penrose pseudoinverse satisfies \((\varepsilon X)^+=\varepsilon^{-1}X^+\), hence
\[
I-(\varepsilon X)^+(\varepsilon X)=I-X^+X.
\]
Therefore every projector \(P_t\) is unchanged by the rescaling, and so are all statistics built only from the projector law.

On the other hand, for any iterate \(w_k\),
\[
\ell_t^{(\varepsilon)}(w_k)
=
\frac{1}{n_t}\|\varepsilon X_t w_k-\varepsilon y_t\|_2^2
=
\varepsilon^2 \ell_t(w_k).
\]
Averaging over \(t\) and expectations gives
\[
F^{\Pi^{(\varepsilon)}}(k)=\varepsilon^2 F^\Pi(k).
\]
\end{proof}

The quantities derived from $S_\Pi$ and $\Pbar$ summarize the projector law at different levels:
\[
S_\Pi \;\Longrightarrow\;
\begin{cases}
\{f_m\}_{m\ge 0},\\[2pt]
\Pbar \;\Longrightarrow\; (f_1,\nu_\Pi).
\end{cases}
\]
The exact identity involves \(S_\Pi^{k-t}(C_t)\,P_t\,S_\Pi^{t-1}(X_\star)\,P_t\), so it also depends on task covariance and target direction. These terms determine the loss scale and spectral contributions.

\section{Recovery through the Task Distribution}\label{app:recovery}

This appendix follows the prescribed-task protocol of Section~\ref{sec:recovery}: learn $T$ independent tasks from $\Pi$, fit a fixed task $\tau$, and resume independent sampling from $\Pi$. We prove Theorem~\ref{thm:distribution-recovery}, derive a complementary upper bound and a cumulative-loss bound, and give a counterexample showing why the current error fraction is needed. We then prove Theorem~\ref{thm:peak-duration} and describe the role of variation between background tasks.

Throughout this appendix, write $C=C_\tau$, $H=H_\Pi$, $S=S_\Pi$, $M=I-H=\mathbb E P$, $c=\|C\|_{\mathrm{op}}$, and $e_m=w_{T+1+m}-w^\star$. We use $0<\varepsilon<1$ and $G_m>0$ whenever a recovery time relative to $G_m$ is considered.

\subsection{Proof of Theorem~\ref{thm:distribution-recovery}}\label{app:proof-recovery}

The prescribed task is independent of the sampled prehistory. Thus the second moment immediately after fitting it is $A_0=P_\tau S^T(X_\star)P_\tau$, and the independence of subsequent tasks gives
\[
A_m=\mathbb E[e_me_m^\top]=S^m(A_0),\qquad G_m=\tr(CA_m),\qquad E_m=\tr A_m.
\]
This conditional form of Theorem~\ref{thm:exact-forgetting-identity} also gives $CA_0=A_0C=0$, since $CP_\tau=0$.

\paragraph{Recovery on the observable space.}
Let $R$ denote the orthogonal projector onto $V_\Pi$. Every background projector fixes $V_\Pi^\perp$ and preserves $V_\Pi$ almost surely. Since $\Range(C)\subseteq V_\Pi$, only $Re_m$ contributes to $G_m$. The positive gap $g_\Pi=\lambda_{\min}(H|_{V_\Pi})$ from Corollary~\ref{cor:observable-decay} gives
\[
\mathbb E[\|Re_{m+1}\|^2\mid e_m]
=\|Re_m\|^2-(Re_m)^\top H(Re_m)
\le(1-g_\Pi)\|Re_m\|^2.
\]
It follows that
\begin{equation}\label{eq:recovery-upper-envelope}
G_{m+s}\le c(1-g_\Pi)^s\mathbb E\|Re_m\|^2
\le c(1-g_\Pi)^sE_m.
\end{equation}
This proves $G_m\to0$ and finiteness of every recovery time in the theorem, including when $\Pi$ assigns no probability mass to $\tau$.

\paragraph{A finite-time lower envelope.}
Future tasks are independent of $e_m$, so $\mathbb E[e_{m+s}\mid e_m]=M^se_m$. Conditional Jensen implies
\[
G_{m+s}\ge\mathbb E\|C^{1/2}M^se_m\|^2.
\]
The reverse triangle inequality in the space of square-integrable random vectors gives
\[
\sqrt{G_{m+s}}\ge\sqrt{G_m}
-\left(\mathbb E\|C^{1/2}(I-M^s)e_m\|^2\right)^{1/2}.
\]
For every $x\in[0,1]$ and integer $s\ge1$,
$[1-(1-x)^s]^2\le1-(1-x)^s\le sx$.
Since $0\preceq H\preceq I$, functional calculus therefore gives $(I-M^s)^2\preceq sH$ and
\[
\begin{aligned}
\|C^{1/2}(I-M^s)\|_{\mathrm{op}}^2
&=\|C^{1/2}(I-M^s)^2C^{1/2}\|_{\mathrm{op}}\\
&\le s\|C^{1/2}HC^{1/2}\|_{\mathrm{op}}=s h_\tau.
\end{aligned}
\]
Consequently the squared displacement term is at most $s h_\tau E_m$. Taking the positive part and squaring proves~\eqref{eq:recovery-displacement}; for $s=0$ it is an equality.

Formally, the recovery time is
\[
R_\varepsilon(m)=\min\{s\in\mathbb N_0:G_{m+t}\le\varepsilon G_m\ \text{for every integer }t\ge s\}.
\]
If $G_{m+s}\le\varepsilon G_m$ and $G_m>0$, the lower envelope requires
$s h_\tau E_m\ge(1-\sqrt\varepsilon)^2G_m$.
Applying this at $s=R_\varepsilon(m)$ and taking the integer ceiling proves~\eqref{eq:recovery-time-lower}. This argument allows arbitrary intervening increases and does not require a single peak.

\subsection{Complementary upper bound and recovery scale}\label{app:recovery-upper}
Set $r_m=G_m/(cE_m)\in(0,1]$ and $\bar h_\tau=h_\tau/c$. For $0<g_\Pi<1$, the monotone upper envelope~\eqref{eq:recovery-upper-envelope} yields
\begin{equation}\label{eq:recovery-sandwich}
\left\lceil\frac{(1-\sqrt\varepsilon)^2r_m}{\bar h_\tau}\right\rceil
\le R_\varepsilon(m)
\le\left\lceil\frac{\log(1/(\varepsilon r_m))}{-\log(1-g_\Pi)}\right\rceil.
\end{equation}
If a family of problems satisfies $r_m\ge r_0>0$ and $\bar h_\tau\le K g_\Pi$ with fixed $r_0,K$, then for fixed $\varepsilon$,
\[
\frac{(1-\sqrt\varepsilon)^2r_0}{\bar h_\tau}
\le R_\varepsilon(m)
\le\frac{K\log(1/(\varepsilon r_0))}{\bar h_\tau}+1.
\]
Thus recovery takes $\Theta(1/\bar h_\tau)$ updates as coverage tends to zero. If $g_\Pi=1$, all background projectors vanish on $V_\Pi$ almost surely, so the fitted-task protocol has $G_m=0$ for every $m$.

\subsection{Accumulated loss following a disruption}\label{app:recovery-area}
Summing the decreasing lower envelope and bounding the integer sum below by its integral gives
\begin{equation}\label{eq:recovery-tail-area}
\sum_{s\ge0}G_{m+s}
\ge\int_0^\infty\left(\sqrt{G_m}-\sqrt{x h_\tau E_m}\right)_+^2\,dx
=\frac{G_m^2}{6h_\tau E_m}.
\end{equation}
The first disruption itself is fixed by the history coupling in~\eqref{eq:conditional-recovery-identity}. With $B_T=S^T(X_\star)$ and $Q=I-P$,
\[
G_1=\mathbb E\|C^{1/2}QP_\tau B_T^{1/2}\|_F^2,\qquad
E_1=\tr[(I-H)P_\tau B_TP_\tau].
\]
If $G_1>0$, substitution into~\eqref{eq:recovery-time-lower} and~\eqref{eq:recovery-tail-area} connects the damage induced by the past and the first new task to its subsequent duration and cumulative cost.

\subsection{Why the current error fraction is necessary}\label{app:recovery-fraction}
Consider coordinate tasks $e_1,e_2,e_3$ with probabilities $\delta,(1-\delta)/2,(1-\delta)/2$, where $0<\delta<1/3$. Let the prescribed direction and its post-fit error be
\[
u=\sqrt{1-\delta}\,e_1+\sqrt{\delta/2}(e_2+e_3),\qquad
e_0=(e_2-e_3)/\sqrt2.
\]
They satisfy $u^\top e_0=0$, and this is a legal state after fitting $u$ (for example with $T=0$ and $w^\star=-(u+e_0)$). With $C=uu^\top$ and $a=(1+\delta)/2$, direct coordinate projection gives
\[
h_\tau=\tfrac32\delta(1-\delta),\qquad
G_m=\frac\delta2(a^m-\delta^m),\qquad E_m=a^m.
\]
The peak is at $m=1$: $G_{m+1}-G_m=\frac\delta2(1-\delta)(\delta^m-a^m/2)$ is negative for $m\ge1$. As $\delta\to0$, $h_\tau\to0$, but $G_{1+s}/G_1\to2^{-s}$ for every fixed $s$, and the recovery time to any fixed fraction remains bounded. Here $r_1=\delta(1-\delta)/(2(1+\delta))\to0$. The prescribed task has a large component along a weakly covered direction, whereas the post-fit residual lies in frequently learned directions.

\subsection{Proof of Theorem~\ref{thm:peak-duration}}\label{app:proof-peak-duration}

For the two background inputs, $H=\operatorname{diag}(h,1-h)$ and
\[
P_\pm=\begin{pmatrix}1-h&\mp\sqrt{h(1-h)}\\\mp\sqrt{h(1-h)}&h\end{pmatrix}.
\]
Every error immediately after fitting $\tau$ lies in $\operatorname{span}(e_2)$, so $A_0=E_0e_2e_2^\top$. Averaging the two signs preserves diagonal second moments. Write $A_m=\operatorname{diag}(a_m,b_m)$. The conditional identity becomes
\[
\begin{pmatrix}a_{m+1}\\b_{m+1}\end{pmatrix}
=K_h\begin{pmatrix}a_m\\b_m\end{pmatrix},\qquad
K_h=\begin{pmatrix}(1-h)^2&h(1-h)\\h(1-h)&h^2\end{pmatrix},\qquad
\begin{pmatrix}a_0\\b_0\end{pmatrix}=\begin{pmatrix}0\\E_0\end{pmatrix}.
\]
Since $K_h$ is the outer product of $(1-h,h)^\top$ with itself,
$K_h^2=\rho_hK_h$ where $\rho_h=(1-h)^2+h^2=1-2h(1-h)\in(0,1)$.
For $m\ge1$ this gives
\[
G_m=a_m=E_0h(1-h)\rho_h^{m-1},\qquad
E_m=a_m+b_m=E_0h\rho_h^{m-1}.
\]
The loss is maximal at $m=1$ and decreases geometrically thereafter. Solving $\rho_h^s\le\varepsilon$ gives the stated recovery time, while geometric summation gives
\[
\sum_{m\ge0}G_m
=\frac{E_0h(1-h)}{1-\rho_h}=\frac{E_0}{2}.
\]
As $h\to0$, $-\log\rho_h\sim2h$, so for fixed $\varepsilon$ the recovery time is asymptotic to $\log(1/\varepsilon)/(2h)$. Since $h_\tau=h$ and $G_1/E_1=1-h$, this family attains the inverse-coverage order in Theorem~\ref{thm:distribution-recovery}. For $h>0$, the cumulative loss equals $E_0/2$; at $h=0$, the tasks are orthogonal and forgetting vanishes identically.

\subsection{Mean updates and variation between tasks}\label{app:recovery-variation}
The conditional identity also has the following nonnegative expansion. For any background distribution, let $M=\mathbb E P$ and define
\[
\mathcal V(A)=\mathbb E[(P-M)A(P-M)],\qquad S(A)=MAM+\mathcal V(A).
\]
For $A\succeq0$, $\mathcal V(A)\succeq0$, and iterating the recursion gives
\[
A_m=M^mA_0M^m+
\sum_{j=0}^{m-1}M^{m-1-j}\mathcal V(A_j)M^{m-1-j}.
\]
Every summand makes a nonnegative contribution to $\tr(CA_m)$. In the present family $M=\operatorname{diag}(1-h,h)$, so the mean residual along $e_1$ is zero at every time, although its mean squared loss is positive. If more generally $C=uu^\top$ and $Hu=hu$, the exact relation is
\[
G_m=\sum_{j=0}^{m-1}(1-h)^{2(m-1-j)}\tr[C\mathcal V(A_j)],
\]
because $G_0=0$. It separates the error introduced by task-to-task variation from its subsequent decay under the mean update.

\section{Numerical Results and Reproducibility}\label{app:reproducibility}

The synthetic experiments examine the decay bounds and task-coverage analysis from Section~\ref{sec:rates}. The image experiments examine how task grouping and representation affect historical loss under the same exact-fit protocol.

\subsection{Synthetic experiments}\label{app:synthetic-experiments}

We use two realizable linear task families to compare the predicted contraction rate with historical forgetting. In the first, the probability of leaving the target direction unobserved varies. In the second, each task makes the same angle with the target, while the number of directions covered by the distribution varies. Both families use $d=192$, rank-$48$ tasks, and independent sampling.

\begin{figure}[htbp]
  \centering
  \begin{subfigure}[t]{0.49\linewidth}
    \centering
    \small\textbf{(a) Bound comparison.}\phantomcaption\label{fig:experiment-bound-compare}\par\vspace{0.2em}
    \includegraphics[width=\linewidth]{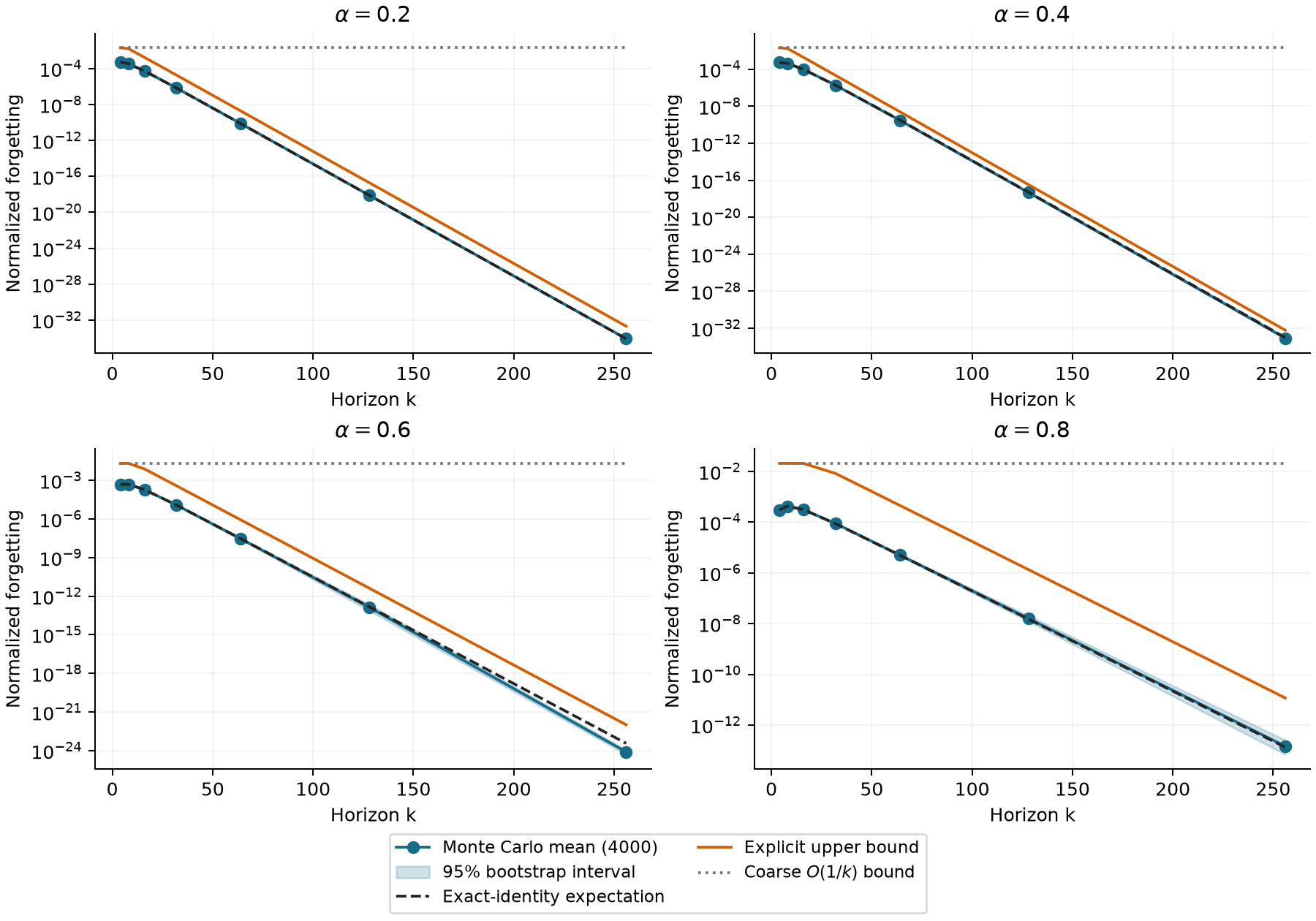}
  \end{subfigure}\hfill
  \begin{subfigure}[t]{0.49\linewidth}
    \centering
    \small\textbf{(b) Rate comparison.}\phantomcaption\label{fig:experiment-rho-comparison}\par\vspace{0.2em}
    \includegraphics[width=\linewidth]{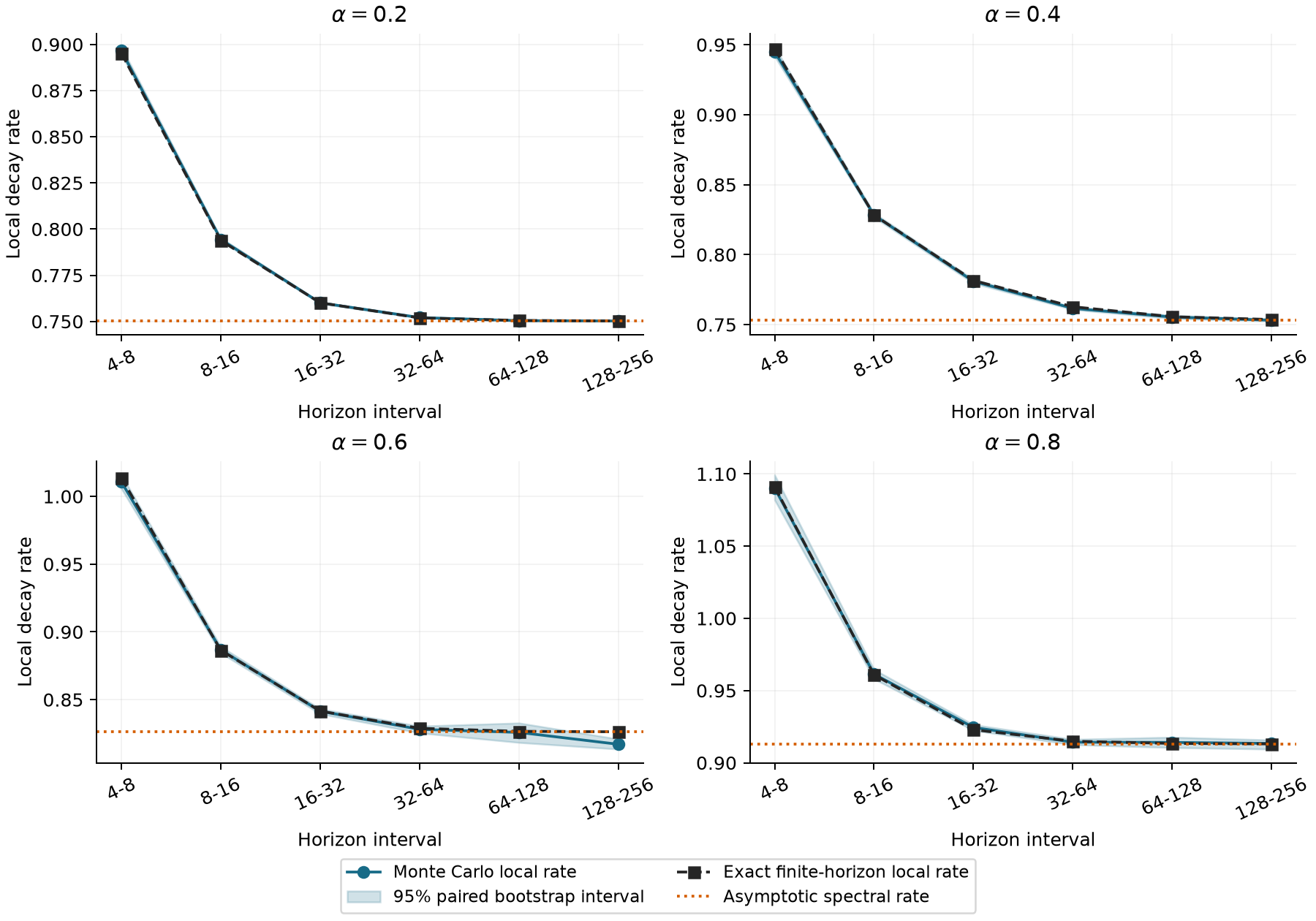}
  \end{subfigure}

  \vspace{0.35em}

  \begin{subfigure}[t]{\linewidth}
    \centering
    \small\textbf{(c) Task-family richness.}\phantomcaption\label{fig:experiment-richness}\par\vspace{0.2em}
    \includegraphics[width=0.78\linewidth]{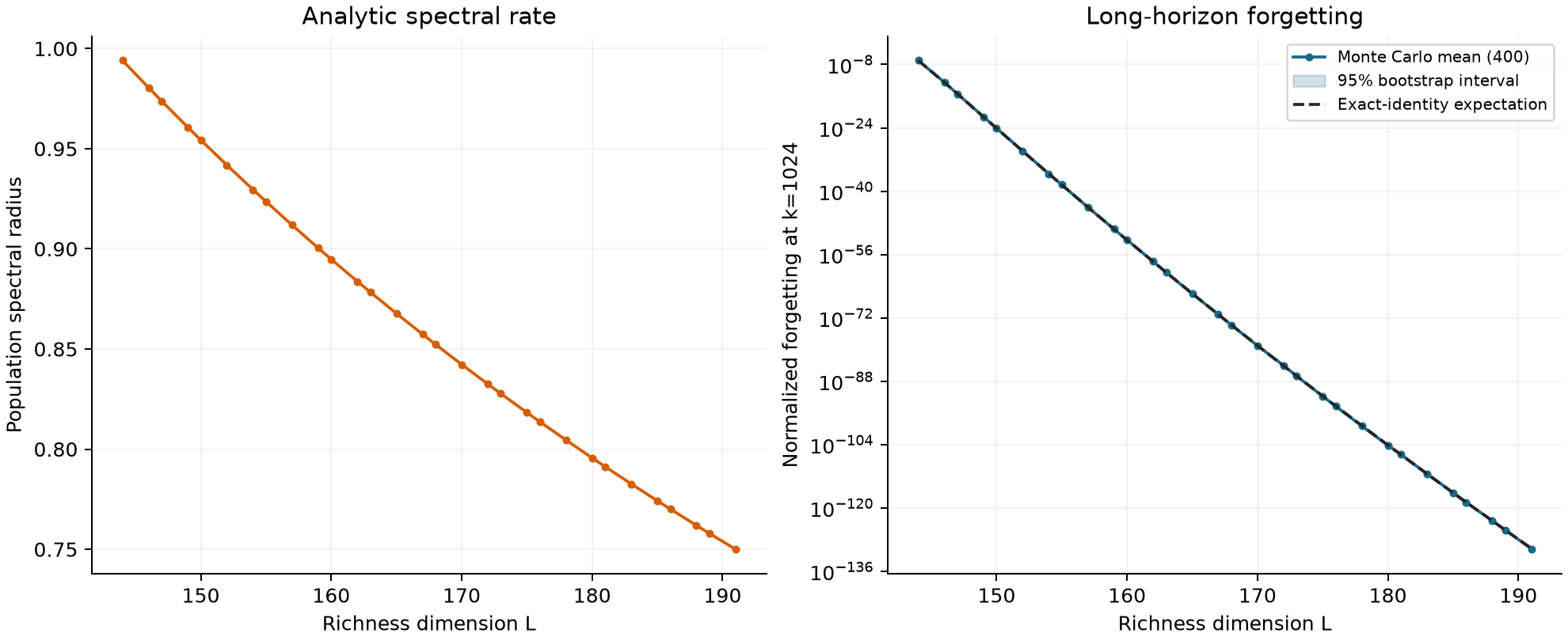}
  \end{subfigure}
  \caption{Synthetic experiments in the realizable exact-fit setting. In the first task distribution, panel (a) compares empirical forgetting and the exact finite-horizon expectation from the spectral expansion with the explicit upper bound and a coarse projector-based \(O(1/k)\) baseline. Both bounds are combined with the trivial bound \(1/r\). Panel (b) compares empirical and exact finite-horizon local decay rates with the analytic \(\rho_\Pi\). In the second task distribution, panel (c) sweeps \(L\) over 30 values: larger \(L\) decreases \(\rho_\Pi\) and lowers long-horizon forgetting. Shaded regions show \(95\%\) bootstrap intervals.}
  \label{fig:experiments-overview}
\end{figure}

Figure~\ref{fig:experiment-bound-compare} shows the explicit bound following the long-horizon exponential decay of the exact expectation. The coarse $O(1/k)$ bound is less informative at these horizons. The projection-residual argument also yields an exponential bound (Corollary~\ref{cor:appendix-exp}), which is not plotted in this panel. In Figure~\ref{fig:experiment-rho-comparison}, the exact local decay rate approaches $\rho_\Pi$ as the horizon grows. Monte Carlo estimates follow this trend, although finite samples produce appreciable deviations at the longest horizons. Figure~\ref{fig:experiment-richness} holds the task--target angle fixed: increasing the number of covered directions reduces $\rho_\Pi$ and lowers long-horizon forgetting.

\subsubsection{Construction and evaluation}

The experiment uses a realizable linear setting with ambient dimension \(d=192\) and task rank \(r=48\). By rotational invariance, we fix a unit vector \(u\in\mathbb R^d\) and set the shared solution to \(w^\star=u\). For each \(\alpha\in\{0.2,0.4,0.6,0.8\}\), the visible-subspace basis \(Q_t\in\mathbb R^{d\times r}\) is sampled independently for every task: with probability \(\alpha\), \(Q_t\) is uniform among rank-\(r\) subspaces of \(u^\perp\), and with probability \(1-\alpha\), it is uniform among rank-\(r\) subspaces of \(\mathbb R^d\). Concretely, we draw a Gaussian matrix in the corresponding space, apply a QR decomposition, and keep the first \(r\) orthonormal columns. All experiments use master seed \(20260924\); the random stream for each trajectory is determined by the task distribution, parameter index, and trajectory index. We then define
\[
X_t = Q_t^\top,\qquad y_t = Q_t^\top w^\star,
\]
so every task shares the same solution \(w^\star\). Equivalently, if \(P_t = I-Q_tQ_t^\top\), then the exact-fit dynamics are
\[
w_t = w^\star + P_t(w_{t-1}-w^\star),\qquad w_0 = 0.
\]
We measure forgetting in this synthetic setting as
\[
F^\Pi(k)=\E\!\left[\frac1k\sum_{t=1}^k \ell_t(w_k)\right],
\qquad
\ell_t(w)=\frac1r\|Q_t^\top(w-w^\star)\|_2^2.
\]

For each $\alpha$, we run $4000$ independent trajectories to step $256$ and record forgetting at $\mathcal K=\{4,8,16,32,64,128,256\}$. \Cref{fig:experiment-bound-compare} compares this curve with the explicit upper bound from Theorem~\ref{thm:upper-bound} and with the coarse \(O(1/k)\) baseline. In this synthetic setting, \(X_\star=w^\star w^{\star\top}\) and \(\|w^\star\|_2=1\), so \(\|X_\star\|_F=1\). Also, \(C_t=\frac1r X_t^\top X_t=\frac1r Q_tQ_t^\top\) is a rescaled rank-\(r\) orthogonal projector, hence \(\|C_t\|_F=1/\sqrt r\). Therefore the explicit upper bound specializes to
\[
U_\Pi(k)=\frac{k-1}{k}\rho_\Pi^{k-1}\frac{1}{\sqrt r}.
\]
The coarse baseline is
\[
B_{\mathrm{coarse}}(k)=\frac{8(d-r)}{r(k-1)},
\]
where the additional factor \(1/r\) comes from the experimental normalization \(\ell_t(w)=\frac1r\|Q_t^\top(w-w^\star)\|_2^2\). Thus this is the projection-based \(O(1/k)\) bound from Appendix~\ref{app:additional-results} expressed at the loss scale used in the experiment. Since each update is an orthogonal projection, \(\|w_k-w^\star\|_2\le1\) and \(F^\Pi(k)\le1/r\). Figure~\ref{fig:experiment-bound-compare} therefore plots \(\min\{U_\Pi(k),1/r\}\) and \(\min\{B_{\mathrm{coarse}}(k),1/r\}\).

Here \(\rho_\Pi\) is computed analytically for each value of \(\alpha\). Let \(m=d-r\), and define
\[
\mu_1:=\frac{m}{d},
\qquad
\mu_2:=\frac{m(m+2)}{d(d+2)},
\qquad
b:=\frac{\mu_1-\mu_2}{d-1},
\]
\[
\beta:=\frac{m-1}{d-1},
\qquad
c_2:=\frac{m(d(m+1)-2)}{d(d+2)(d-1)},
\qquad
\gamma:=\frac{(m-1)((d-1)m-2)}{(d-1)(d+1)(d-2)}.
\]
Then
\[
M_\alpha=
\begin{pmatrix}
\alpha+(1-\alpha)\mu_2 & (1-\alpha)b \\
(1-\alpha)(\mu_1-\mu_2) & \alpha\beta + (1-\alpha)(\mu_1-b)
\end{pmatrix},
\]
and
\[
\lambda_+(M_\alpha):=
\frac{\tr(M_\alpha)+\sqrt{\tr(M_\alpha)^2-4\det(M_\alpha)}}{2},
\]
\[
\lambda_{\mathrm{mix}}:=\alpha\beta+(1-\alpha)c_2,
\qquad
\lambda_{\mathrm{tr}}:=\alpha\gamma+(1-\alpha)c_2.
\]
The analytic rate used in both figures is
\[
\rho_\Pi=\max\bigl\{\lambda_+(M_\alpha),\,\lambda_{\mathrm{mix}},\,\lambda_{\mathrm{tr}}\bigr\}.
\]
\Cref{fig:experiment-rho-comparison} compares this analytic \(\rho_\Pi\) with the empirical local decay rate
\[
\hat\rho_{\mathrm{loc}}
=
\exp\!\left(
\frac{\log \widehat F^\Pi(k_{i+1})-\log \widehat F^\Pi(k_i)}{k_{i+1}-k_i}
\right),
\]
computed from adjacent horizons \(k_i<k_{i+1}\). \Cref{fig:experiment-rho-comparison} also reports the finite-horizon local rate computed from the exact spectral expansion. The \(95\%\) uncertainty intervals for empirical means and local rates are obtained from \(2000\) paired bootstrap resamples of complete trajectories.

For \Cref{fig:experiment-richness}, we consider a second task distribution. We again fix \(d=192\), \(r=48\), and \(w^\star=u\), but now each task null space is constructed as
\[
N_t=\operatorname{span}(n_t)\oplus H_t,
\qquad
n_t=\cos(\theta)\,u+\sin(\theta)\,v_t,
\]
where \(v_t\) is sampled from an \(L\)-dimensional subspace \(R_L\subset u^\perp\), and \(H_t\) is an \((m-1)\)-dimensional random complement inside \(R_L\cap v_t^\perp\), with \(m=d-r\). The visible-subspace basis \(Q_t\) is then sampled as an orthonormal basis of \(N_t^\perp\), and we again set \(X_t=Q_t^\top\), \(y_t=Q_t^\top w^\star\). In the richness figure we fix \(\theta=30^\circ\) and sweep \(30\) values of \(L\) between \(144\) and \(191\). The horizons are \(\{4,8,16,32,64,128,256,512,1024\}\), and each empirical point averages over \(400\) independent task sequences.

Let \(c:=\cos\theta\) and \(s:=\sin\theta\). On the relevant rotation-invariant subspace, the operator \(S_\Pi\) has matrix representation
\[
K_L=
\begin{pmatrix}
c^4 & c^2s^2/\sqrt L \\
c^2s^2/\sqrt L & (m-1+s^4)/L
\end{pmatrix}.
\]
Thus the population spectral rate used in \Cref{fig:experiment-richness} is \(\rho_\Pi=\lambda_{\max}(K_L)\). The exact finite-horizon expectations shown in the figures are computed from Theorem~\ref{thm:exact-spectral-expansion}, rather than estimated from a finite projector sample.

\subsection{Image tasks and paired regrouping}\label{app:image-experiments}\label{sec:image-experiments}

We study how organizing the same labeled images into different tasks changes historical loss. Tasks contain digits from MNIST, objects from CIFAR-10 and CIFAR-100, or portraits from Yearbook. They group images by class, transformation, or calendar period. We compare each such grouping with a random partition of the same image--label pairs. The predictor is linear in pixels or frozen ResNet-18 features, and each sampled task is fitted exactly. The measured quantity is squared loss on the training history.

\paragraph{Task families and representations.}
Table~\ref{tab:image-settings} lists the seven task families and their representations, giving eleven combinations. MNIST and CIFAR-10 class tasks each contain one class. CIFAR-100 banks select 20 classes and assign successive pairs to 10 tasks. Rotation tasks use four angles ($0^\circ,45^\circ,90^\circ,135^\circ$) or seven angles ($0^\circ,10^\circ,\ldots,60^\circ$), with bilinear interpolation and zero padding. Each rotation or permutation task transforms the same base image set and preserves its labels; the eight permutation tasks use independently generated, fixed permutations of the 784 pixel positions. Yearbook uses eight equal-width calendar bins between 1930 and 2014, with balanced binary labels in each full bank. Tasks are sampled uniformly and independently with replacement. For Yearbook, calendar period determines task membership, while presentation order is random.

\begin{table}[t]
\centering
\caption{Image-task constructions. Every combination uses five banks at $n=3$ and $n=8$ images per task. Larger banks use ten pixel banks or five feature banks. Each bank is evaluated under its original grouping and random regrouping.}
\label{tab:image-settings}
\begin{tabular}{llrr}
\toprule
Task family & Representations & Tasks per bank & Larger-bank $n$ \\
\midrule
MNIST classes & Pixels, ResNet-18 & 10 & 16 \\
MNIST four rotations & Pixels & 4 & 20 \\
MNIST gradual rotations & Pixels & 7 & 20 \\
MNIST permutations & Pixels & 8 & 20 \\
CIFAR-10 classes & Pixels, ResNet-18 & 10 & 16 \\
CIFAR-100 class pairs & Pixels, ResNet-18 & 10 & 16 \\
Yearbook periods & Pixels, ResNet-18 & 8 & 20 \\
\bottomrule
\end{tabular}
\end{table}

Pixel vectors are standardized and then normalized to unit Euclidean norm. MNIST uses mean $0.1307$ and standard deviation $0.3081$; CIFAR-10 uses channel means $(0.4914,0.4822,0.4465)$ and standard deviations $(0.2470,0.2435,0.2616)$; CIFAR-100 uses $(0.5071,0.4867,0.4408)$ and $(0.2675,0.2565,0.2761)$. Yearbook grayscale images are resized to $32\times32$ and scaled by mean and standard deviation $0.5$. Feature experiments use frozen ImageNet ResNet-18 weights (torchvision's \texttt{IMAGENET1K\_V1}) with the classifier removed. We apply the weight package's standard resize, crop, and normalization, replicate grayscale inputs across three channels, and normalize the resulting 512-dimensional features to unit norm. Paired pixel and feature experiments use the same source images, labels, and task membership.

\paragraph{Shared solution and exact updates.}
Let $X$ stack the image vectors in a bank and let $Y$ contain their one-hot labels. The selected banks have full row rank, with fewer rows than features, so $\Theta_\star=X^\dagger Y$ fits all labels. Computation uses an orthonormal basis of the row span of $X$, preserving its Gram matrix and predictions without whitening. Starting from $\Theta_0=0$, fitting task $i$ gives
\[
\Theta_k=\Theta_{k-1}+X_i^\dagger(Y_i-X_i\Theta_{k-1}),\qquad
\ell_i(\Theta)=\frac1{n_i}\|X_i\Theta-Y_i\|_F^2.
\]
This multi-output loss sums the scalar losses in the theory, so the identity applies with $X_\star=\Theta_\star\Theta_\star^\top$. There is no division by the number of classes. Each task has loss one at initialization, and $F^\Pi(k)$ averages over task occurrences, including the zero loss of the task just fitted.

\paragraph{Paired controls and evaluation scope.}
For each bank, the control uniformly permutes all rows and repartitions them into the same number of equally sized tasks, using seed $44000+s$ for bank seed $s$. This control preserves the labeled images, representation, common interpolator, and pooled covariance:
\[
\sum_i p_iC_i=\frac1N X^\top X.
\]
Regrouping changes individual task projectors and generally changes $S_\Pi$, whereas Proposition~\ref{prop:main-same-operator} compares laws with the entire operator fixed. Across the 110 primary paired comparisons, the largest entrywise pooled-covariance discrepancy is $2.22\times10^{-16}$ and the largest joint-fit residual is $3.11\times10^{-15}$.

The primary evaluation uses five banks (seeds $0$--$4$), two task sizes ($n=3,8$), and both groupings for each of the eleven combinations, giving 220 distributions. The smaller banks are nested, label-stratified subsets of the larger source banks. Within each original task, seed $71100+s$ randomizes label order and images within each label, and selection proceeds round-robin across labels. Image selection for source banks uses seed $2026092400+s$. We also evaluate all larger banks: seven pixel families with ten banks each and four feature families with five banks each, under both groupings, giving 180 additional curves. All 400 curves are evaluated through step 1024.

\paragraph{Exact expectation and numerical checks.}
For each finite distribution, we diagonalize $S_\Pi$ on symmetric matrices, combine eigenvectors from the same spectral level, and evaluate Theorem~\ref{thm:exact-spectral-expansion}. These are numerical evaluations of the expectation over all i.i.d.\ task sequences. An independent recurrence initializes $B_0=X_\star$, $D_{j,0}=0$ and uses
\[
B_k=\sum_i p_iP_iB_{k-1}P_i,\qquad
D_{j,k}=\sum_i p_iP_iD_{j,k-1}P_i+p_jP_jB_{k-1}P_j,
\qquad
F^\Pi(k)=\frac1k\sum_j\tr(C_jD_{j,k}).
\]
Here $D_{j,k}$ weights the error second moment by the number of occurrences of task $j$ in the history. The recurrence and spectral formula differ by at most $2.02\times10^{-13}$ over the first 64 steps of the primary evaluation.

\subsection{Effects of grouping and representation}

Figure~\ref{fig:image-single-peak-main} shows three examples: class grouping increases forgetting on MNIST pixels and CIFAR-10 features, while permutation grouping reduces it. Each panel includes all five banks, with seed zero highlighted. The task families illustrate different directions of the effect; Table~\ref{tab:image-effects} and Figure~\ref{fig:image-effects} report all families, including weaker effects and changes in direction.

\begin{figure}[t]
\centering
\includegraphics[width=\textwidth]{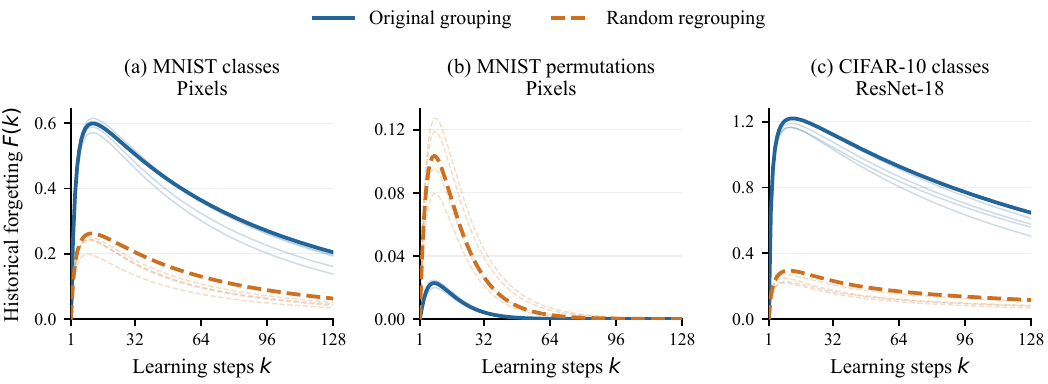}
\caption{Historical loss under original grouping and random regrouping of the same labeled images. Each task contains eight images. Bold curves use bank seed zero, and faint curves show the other four banks; all thirty curves are shown. Losses use the original scale, and vertical scales differ across panels.}
\label{fig:image-single-peak-main}
\end{figure}

\paragraph{Representation changes the grouping effect.}
On the same CIFAR-10 images, class grouping has a larger effect with frozen ResNet-18 features than with pixels. Across five matched larger banks, the median original-to-random peak ratio is $5.55$ in features and $1.12$ in pixels. Each ratio compares groupings within a fixed representation.

\paragraph{Contraction and finite-horizon loss can give different rankings.}
For MNIST class tasks, all ten larger pixel banks have a smaller observable spectral radius under class grouping but a higher forgetting peak. The radius describes contraction; the magnitude of historical loss also depends on task covariances and target-dependent coefficients in the identity. The comparison can also depend on the horizon: in the larger four-angle rotation banks, original-to-random peak ratios range from $0.83$ to $1.03$, while original grouping gives lower loss at step 64 in all ten banks, with a median ratio of $0.47$.

\paragraph{The direction of the effect depends on the task family.}
Gradual rotations have a median peak ratio of $0.82$ at $n=8$ and $1.16$ at $n=20$. Yearbook pixel banks have ratios on both sides of one at both sizes. Original grouping therefore does not uniformly increase or decrease forgetting across these families.

\begin{table}[t]
\centering
\caption{Paired original-to-random ratios, reported as medians across banks. The full-bank peak column also gives the range. Values summarize five banks at $n=8$, ten larger pixel banks, and five larger feature banks.}
\label{tab:image-effects}
\begin{tabular}{llrrrr}
\toprule
& & \multicolumn{2}{c}{$n=8$} & \multicolumn{2}{c}{Larger source banks} \\
Task family & Representation & Peak & $F(64)$ & Peak [min, max] & $F(64)$ \\
\midrule
MNIST classes & Pixels & 2.42 & 3.41 & 2.74 [2.58, 2.92] & 3.93 \\
MNIST four rotations & Pixels & 0.61 & 0.13 & 0.93 [0.83, 1.03] & 0.47 \\
MNIST gradual rotations & Pixels & 0.82 & 0.85 & 1.16 [1.05, 1.39] & 1.36 \\
MNIST permutations & Pixels & 0.22 & 0.06 & 0.33 [0.31, 0.36] & 0.05 \\
CIFAR-10 classes & Pixels & 1.12 & 1.33 & 1.13 [1.10, 1.19] & 1.35 \\
CIFAR-100 class pairs & Pixels & 1.10 & 1.27 & 1.13 [1.07, 1.23] & 1.31 \\
Yearbook periods & Pixels & 1.16 & 1.00 & 1.19 [0.97, 1.46] & 1.39 \\
MNIST classes & ResNet-18 & 4.61 & 8.17 & 5.99 [5.06, 6.33] & 11.80 \\
CIFAR-10 classes & ResNet-18 & 4.87 & 7.22 & 5.55 [5.17, 6.23] & 9.79 \\
CIFAR-100 class pairs & ResNet-18 & 2.47 & 3.45 & 3.08 [3.01, 3.21] & 4.96 \\
Yearbook periods & ResNet-18 & 1.11 & 1.41 & 1.16 [1.10, 1.83] & 1.29 \\
\bottomrule
\end{tabular}
\end{table}

\paragraph{Reproducibility.}
The source package contains the figure assets, saved image-task curves and paired summaries, and scripts for checking the two-dimensional recovery formulas and redrawing the representative image comparison. The image results describe historical training loss under fixed representations and exact fitting. Because subsequent sampling can revisit every task in each bank, these experiments study how the task distribution affects historical loss. Recovery without exact revisits is established by the theory in Section~\ref{sec:recovery}.

\begin{figure}[htbp]
\centering
\includegraphics[width=\textwidth]{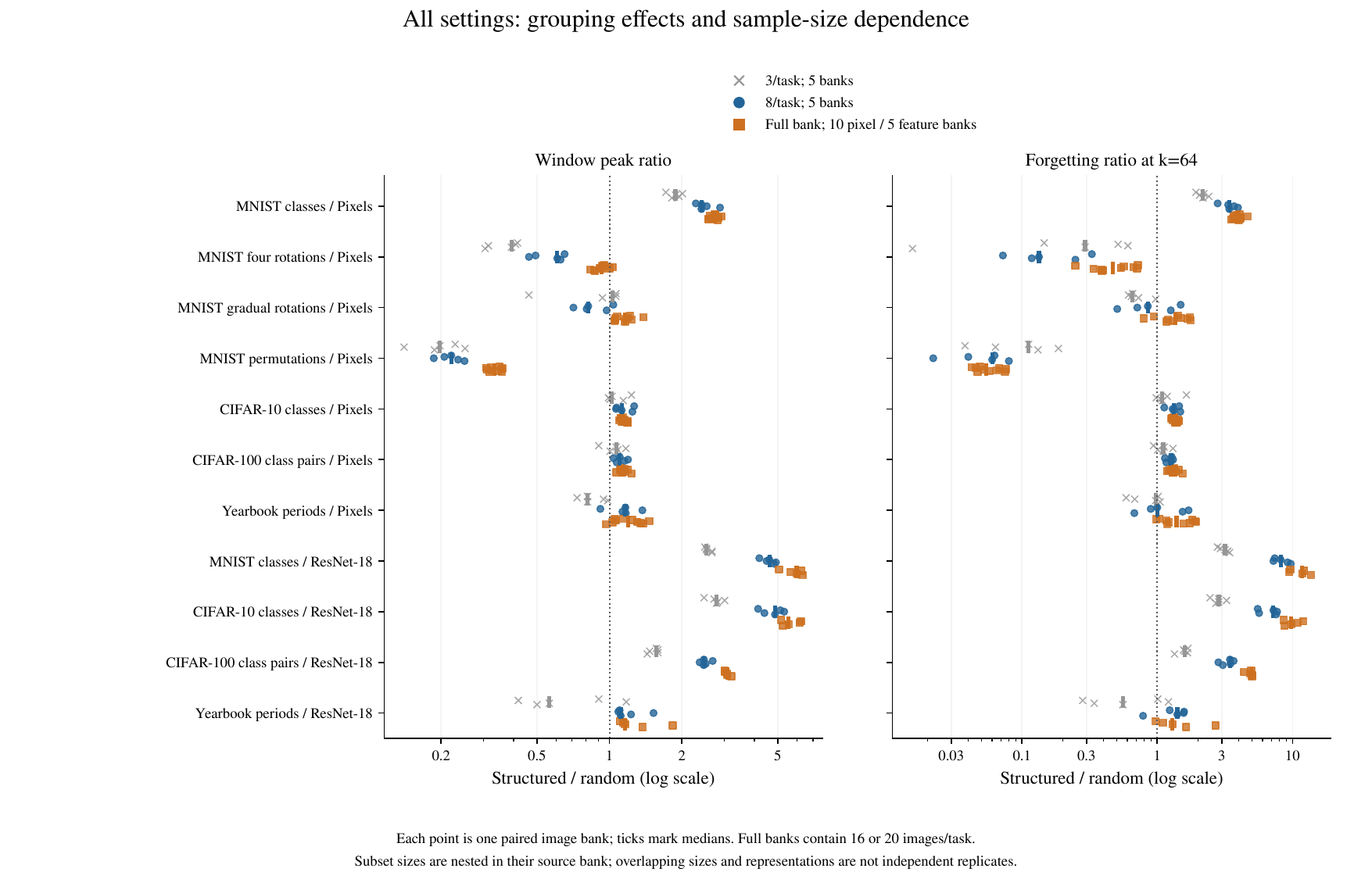}
\caption{Paired grouping effects across every family and representation. Each point is one image bank; ratios above one indicate more forgetting under original grouping. The three marker types show $n=3$, $n=8$, and the larger source banks. Peak height and loss at $k=64$ are compared separately. The vertical reference line marks equality.}
\label{fig:image-effects}
\end{figure}

\FloatBarrier

\end{document}